\documentclass[journal]{IEEEtran}
\usepackage {threeparttable}
\usepackage{amsthm}
\usepackage{flushend}
\usepackage{multirow}
\usepackage{subfigure}
\usepackage{makecell}
\usepackage{booktabs}
\usepackage{array}
\usepackage{float}
\usepackage{amsmath}
\usepackage{amsfonts}
\usepackage{epsfig}
\usepackage{graphicx}

\usepackage{arydshln}
\usepackage{algpseudocode}
\usepackage{verbatim}
\usepackage{enumerate}
\usepackage{rotating}\usepackage{enumitem}
\usepackage{color}
\usepackage{caption} 
\usepackage{mathrsfs}
\usepackage{amssymb}
\usepackage{mathtools}
\usepackage[linesnumbered,ruled,vlined]{algorithm2e}
\usepackage{cite}
\usepackage{float}
\usepackage{soul}
\usepackage{tikz}
\usepackage{xcolor}
\usepackage{hyperref}
\hypersetup{hidelinks} 
\usepackage{soul}
\usepackage{booktabs}
\usepackage{multirow}
\usepackage{makecell}
\usepackage{array}
\usetikzlibrary{arrows.meta}
\usetikzlibrary{patterns}
\usepackage{tcolorbox}
\usepackage{fancyhdr}
\usepackage{flushend}
\allowdisplaybreaks[4]
\newtheorem{mydef}{\bf Definition}
\newtheorem{mythm}{\bf Theorem}
\newtheorem{myprob}{\bf Problem}

\newtheorem{mypro}{\bf Proposition}

\newtheorem{remark}{Remark}

\usepackage{amsfonts}
\DeclareMathSymbol{\shortminus}{\mathbin}{AMSa}{"39}

\floatname{algorithm}{Procedure}

\title{Signal Temporal Logic Motion Planning\\  via Graphs of Convex Sets}

  \author{Yu Chen, Ancheng Hou, Mingyang Feng, Xiao Yu and Xiang Yin
 \thanks{This work was supported by  the National Natural Science Foundation of China (62061136004, 62173226, 61833012).}
 	\thanks{Yu Chen, Ancheng Hou, Mingyang Feng and Xiang Yin are with School of Automation \& Intelligent Sensing, Shanghai Jiao Tong University, Shanghai 200240, China. Xiao Yu is with Institute of Artificial Intelligence, Xiamen University, Xiamen, China.
  	{\tt\small \{yuchen26, hou.ancheng, Fmy-135214, yinxiang\}@sjtu.edu.cn, xiaoyu@xmu.edu.cn}.}
 }
\begin{document}
	
\maketitle
	\thispagestyle{empty}
	\pagestyle{empty}

\begin{abstract}
This paper investigates continuous-time motion planning under Signal Temporal Logic (STL) specifications.
The goal is to generate smooth robot trajectories that satisfy high-level logical and timing requirements while respecting low-level motion constraints.
To this end, we propose an efficient framework that combines timed-automata reasoning with graphs of convex sets (GCS).
An STL specification is first represented by a timed automaton, which is then coupled with a convex decomposition of the configuration space to form a joint transition system encoding both task progress and region occupancy.
Based on this joint transition system, the STL motion-planning problem is reformulated as a shortest-path problem over a GCS, whose solution induces a smooth B\'ezier-spline trajectory satisfying the STL specification, smoothness requirements, and velocity bounds.
We establish the soundness of the proposed formulation and analyze its computational complexity, showing that, once the timed automaton and convex decomposition are fixed, the convex relaxation scales polynomially with the configuration-space dimension and the B\'ezier degree.
We further develop a compact timed-automaton construction for an expressive STL fragment using dedicated templates and Boolean composition.
Numerical experiments on low-dimensional benchmarks, a $3$-D quadrotor, a $30$-DoF humanoid, and a hardware experiment on a UR-3 robot arm demonstrate that the proposed method efficiently solves complex STL motion-planning problems and produces smooth executable trajectories.
\end{abstract}

\section{Introduction}
\subsection{Motivation}

Task and motion planning is a fundamental problem in robotics, where a robot must generate physically feasible motions while accomplishing high-level tasks.
Unlike classical motion planning, which mainly focuses on geometric or dynamical feasibility, task and motion planning must also jointly reason about  task-level constraints such as logical ordering and  timing requirements. This coupling makes the problem particularly challenging: low-level motion feasibility is typically defined over a continuous configuration space, whereas high-level task specifications are often discrete, logical, and combinatorial. 

In recent years, formal languages have attracted increasing attention as a principled way to describe high-level robotic tasks ~\cite{kressgazit2018synthesis,belta2019formal,yin2024formal}. Among them, Signal Temporal Logic (STL) has emerged as an expressive specification language for temporal behaviors of real-valued physical signals~\cite{maler2004monitoring,fainekos2009motion}.
STL can naturally encode requirements such as ordered visits, deadlines, dwell-time constraints, persistent surveillance, and safety conditions.
This makes it well suited for robotic applications in which task satisfaction depends not only on where the robot moves, but also on when and for how long certain conditions hold.
Recent applications in aerial robotics~\cite{silano2021power,yang2025stlgame}, robotic manipulation~\cite{puranic2021learning,sewlia2022cooperative,yuasa2026neuro}, and legged locomotion~\cite{gu2025robust} further demonstrate the practical relevance of STL-based planning.

Despite its expressive power, STL motion planning remains difficult to scale.
A common approach is to encode STL constraints over time-discretized trajectories and solve the resulting optimization problem~\cite{raman2014mpc,sadraddini2015robust,farahani2015robust,kurtz2022mixed,cardona2025disjunction,kurtz2020trajectory}.
Such methods are attractive because they can directly incorporate system dynamics, input bounds, and logical constraints within a single optimization problem.
However, they usually lead to large mixed-integer programs, whose complexity grows rapidly with the planning horizon, the number of logical constraints, and the dimension of the system.
Moreover, fine time discretization is often needed to reduce inter-sample violations, such as passing through an obstacle between two sampled states, which  further increases the number of decision variables. Consequently, discretization-based STL planning becomes especially challenging for long-horizon tasks and high-dimensional robotic systems.

An alternative is to decouple the full dynamic optimization from the high-level planning problem by generating continuous  trajectories.
For many robotic systems, especially differentially flat systems, smooth reference trajectories can be executed by low-level controllers and therefore provide a practical route to dynamic feasibility~\cite{murray1995differential}.
Graphs of convex sets (GCS) have recently emerged as a powerful framework for this purpose, combining discrete path selection with continuous trajectory optimization over convex regions~\cite{marcucci2024thesis,marcucci2024shortest,marcucci2023motion}.
GCS-based methods can generate smooth continuous-time trajectories and have shown favorable scalability in high-dimensional motion-planning problems.
However, existing GCS formulations  mainly focused on classical objectives such as obstacle avoidance, and reach-avoid navigation, with limited treatment of rich temporal-logic tasks~\cite{kurtz2023temporal}.
How to retain the scalability and smooth-trajectory advantages of GCS while enforcing the logical and timing requirements of STL specifications remains the key challenge addressed in this paper.

\subsection{Our Contributions} 

In this paper, we  address the task-and-motion-planning problem under signal temporal logic specifications.
Our objective is to synthesize a continuous-time smooth trajectory that satisfies a given STL task while respecting geometric, timing, smoothness, and velocity constraints.
To this end, we develop a unified planning framework that combines timed-automata reasoning with graphs of convex sets.
Starting from an STL specification, we first represent the temporal task by a timed automaton, which captures the logical and timing progress required for task satisfaction.
We then combine this timed automaton with a convex decomposition of the configuration space to construct a joint transition system that records both the task progress and the convex region occupied by the trajectory.
Based on this joint transition system, we build a GCS instance and reformulate STL motion planning as a shortest-path problem over convex sets, where each selected path induces a smooth B\'ezier-spline trajectory satisfying the STL specification, and the imposed motion constraints.

In this way, the proposed framework provides a bridge between formal temporal reasoning and continuous trajectory optimization. The timed automaton handles the logical structure and real-time requirements of STL specifications, while the GCS formulation exploits convex optimization to generate smooth trajectories efficiently over continuous configuration spaces. The main contributions of this paper are summarized as follows:
\begin{itemize}[leftmargin=*,topsep=0pt,itemsep=3pt]
    \item We propose a GCS-based formulation for STL motion planning.
    Given an STL specification represented by a timed automaton, the proposed construction converts logical task satisfaction, region occupancy, smoothness constraints, and velocity bounds into a shortest-path problem over convex sets.
    The resulting solution directly yields a continuous-time B\'ezier-spline trajectory.

    \item In addition to the general formulation, we identify an expressive STL fragment that covers many robotic  tasks and provide a compact timed-automaton construction for this fragment.
    The construction is based on constant-size templates for elementary temporal patterns and union/product operations for Boolean composition, which substantially reduces the complexity compared with classical general-purpose timed-automaton constructions.

    \item We establish the soundness of the proposed GCS formulation, showing that every feasible solution of the constructed GCS problem yields a trajectory satisfying the original STL task together with the smoothness and velocity constraints.
    We also analyze the computational complexity and show that, once the timed automaton and convex decomposition are fixed, the convex relaxation scales polynomially with the configuration-space dimension and the B\'ezier degree.

    \item We conduct extensive numerical and hardware experiments to validate the scalability and practicality of the proposed approach.
    The experiments include low-dimensional STL planning benchmarks, a $3$-D quadrotor example, a $30$-DoF humanoid simulation, and a hardware implementation on a UR-3 robot arm.
    The results show that our method can efficiently solve complex STL motion-planning problems and generate smooth executable trajectories for both low-dimensional and high-dimensional robotic systems.
\end{itemize}
\subsection{Related Works}

\textbf{STL Planning via Optimization-Based Approaches:}
A fundamental approach to STL motion planning is to formulate the problem as a constrained optimization problem.
For systems with time-discretizations, both the system dynamics and the satisfaction of STL formulae can be encoded as algebraic constraints, leading to mixed-integer optimization formulations~\cite{raman2014mpc,raman2015reactive,farahani2015robust,sadraddini2015robust,kurtz2022mixed,cardona2023weighted,cardona2025disjunction}.
These methods provide precise optimization-based formulations for the discretized planning problem and can explicitly incorporate dynamics, input constraints, and logical requirements.
However, they typically suffer from rapidly growing computational complexity as the planning horizon, the system dimension, the number of logical constraints, and the discretization resolution increase.
Moreover, fine time discretization is often needed to avoid inter-sample violations, which further increases the size of the resulting mixed-integer programs.

To mitigate the complexity caused by time discretization, continuous-time optimization methods have also been explored. For example, timed-waypoint and piecewise-linear planning methods can handle complex STL specifications without discretizing the entire trajectory into a large number of time samples~\cite{sun2022multi,sato2024optimization}.
However, these methods usually return piecewise-linear trajectories, which require an additional low-level tracking controller for execution and do not directly provide smooth trajectory parameterizations.
In contrast, our approach formulates STL motion planning as a GCS-based shortest-path problem over convex sets. The resulting solution directly yields a smooth B\'ezier-spline trajectory, while the associated convex relaxation scales polynomially with the B\'ezier degree and the configuration-space dimension once the timed automaton and convex decomposition are fixed.

\textbf{STL Planning via Timed Automata:}
Automata-based methods have also been widely studied for motion planning and control under timed temporal logics such as MITL and STL. Foundational results in formal methods show that timed specifications can be translated into timed automata or related automaton models~\cite{alur1994theory,alur1996benefits,brihaye2017mightyl,ferrere2019real}.
The resulting automata encode the temporal progress of the specification and can then be used as the basis for planning  under timed-logic constraints~\cite{fu2015computational,zhou2016timed,nikou2016cooperative,lindemann2020efficient}.
However, most  existing automata-based planning methods  require a finite abstraction of the underlying dynamics. Constructing such abstractions becomes increasingly expensive as the system dimension and resolution grow, which limits their scalability to high-dimensional robotic systems.
In contrast, our approach avoids building a global abstraction and directly combines timed-automaton reasoning with GCS-based continuous trajectory optimization.
Furthermore, although our framework can use timed automata obtained from standard constructions such as~\cite{alur1996benefits}, we also provide a more compact construction for an expressive STL fragment, where  elementary temporal formulae are handled by dedicated TA templates, and compositional constructions are needed only for conjunction and disjunction.
This avoids recursive temporal-operator composition and significantly reduces the automaton size for the considered class of robotic tasks.

\textbf{Sampling-Based STL Planning:}
Sampling-based methods constitute another important class of motion-planning algorithms, with RRT~\cite{lavalle2001randomized} and its asymptotically optimal variant RRT$^\ast$~\cite{karaman2011sampling} being representative examples.
Recent works have extended this paradigm to STL planning by incorporating MILP encodings~\cite{nikou2016cooperative}, STL robustness measures~\cite{vasile2017sampling,ahmad2026rrt}, control barrier functions~\cite{marchesini2026sampling}, or automata-theoretic guidance~\cite{ho2022automaton} into the sampling process.
These methods are attractive because they are broadly applicable and can provide probabilistic completeness or asymptotic optimality guarantees under suitable assumptions.
However, it is well known that sampling-based motion planners often suffer from poor practical scalability in high-dimensional configuration spaces, where the number of samples required to find a feasible or high-quality solution can grow rapidly. This difficulty becomes even more pronounced for STL planning, since the planner must satisfy not only geometric feasibility but also logical ordering and timing constraints.
Moreover, sampling-based planners typically return discrete waypoints or piecewise-linear paths, and additional smoothing, post-processing, or tracking control is often needed to obtain executable trajectories.
In contrast, our GCS-based approach constructs a convex decomposition of the configuration space and then solves an optimization problem over the resulting graph of convex sets.
This allows us to directly generate continuous-time B\'ezier trajectories and explicitly impose smoothness, velocity, and timing constraints within the planning formulation.

\textbf{Learning-Based STL Planning:}
Learning-based approaches have also been explored for STL tasks.
Reinforcement learning methods typically incorporate STL robustness or progress measures into rewards or value functions and then learn policies from interaction or data~\cite{aksaray2016q,balakrishnan2019structured,venkataraman2020tractable,kalagarla2021model,ikemoto2022deep,wang2024synthesis}.
More recently, diffusion-based methods have been used to generate trajectories by sampling from learned models guided by STL-related criteria~\cite{zhong2023guided,meng2024diverse,liu2025zeroshot,liu2026dagstl}.
These approaches are attractive when accurate system models are unavailable or when data-driven generalization is desired.
However, they generally do not provide the same formal task-completion guarantees as synthesis-based methods, since STL satisfaction is usually encouraged through learned rewards, robustness surrogates, or sampling guidance rather than enforced by an explicit formal construction.
By contrast, our method uses timed automata and GCS constraints to enforce STL satisfaction in a sound optimization-based framework.

\textbf{Temporal Logic  Planning via GCS:}
Graphs of convex sets have recently emerged as a powerful framework for continuous motion planning~\cite{marcucci2024thesis,marcucci2024shortest,marcucci2023motion}.
By associating graph vertices with convex sets and edges with convex constraints, GCS formulations combine discrete path selection with continuous trajectory optimization.
Recent works have begun to extend GCS-based planning toward temporal-logic specifications.
The work~\cite{kurtz2023temporal} is particularly close in spirit to ours.
It studies motion planning under LTL specifications by formulating the problem as a shortest-path problem over a GCS and exploiting the strong convex relaxation.
Compared with~\cite{kurtz2023temporal}, our focus is on STL specifications with explicit real-time intervals and continuous-time semantics.
This requires a dedicated treatment of clock variables, timed progress, and STL satisfaction through timed automata.
Another closely related work is~\cite{lin2025towards}, which considers STL planning for discrete-time piecewise-affine systems and, inspired by GCS, proposes a logic-network-flow formulation to obtain tighter convex relaxations than logic-tree-based mixed-integer encodings.
In contrast, our framework directly constructs a GCS formulation for continuous-time STL motion planning.
The selected GCS path induces a smooth B\'ezier-spline trajectory, and the timing constraints are enforced through the coupled timed-automaton and GCS construction.
Thus, our approach combines the scalability of GCS-based continuous trajectory optimization with the formal expressiveness of STL specifications.

\subsection{Organization}

The remainder of this paper is organized as follows. After introducing the necessary background in Section~\ref{sec:prelinimary}, Section~\ref{sec:problemformulation} formulates the considered STL motion-planning problem. 
Section~\ref{sec:GCSforSTL} then shows how to convert STL planning into a GCS-based optimization problem, while Section~\ref{sec:timeautomaton} develops an efficient timed-automaton construction for an expressive STL fragment. 
Experimental results are reported in Section~\ref{sec:simulation}, and concluding remarks are given in Section~\ref{sec:con}.

\section{Preliminaries}\label{sec:prelinimary}

\textbf{Notation:}
We use $\mathbb{R}$, $\mathbb{R}_{\geq 0}$, and $\mathbb{R}^n$ to denote the sets of real numbers, nonnegative real numbers, and $n$-dimensional real vectors, respectively.
The set of nonnegative integers is denoted by $\mathbb{Z}_{\geq 0}$.
For a finite set $A$, its power set is denoted by $2^A$.
Let $T \in \mathbb{R}_{\geq 0}$ be the time horizon.
For $d,n \in \mathbb{Z}_{\geq 0}$ with $n>0$, we define
\[
\mathcal{C}^d_n(T)
=
\{ \xi:[0,T]\to \mathbb{R}^n \mid \xi \text{ is $d$ times conti. differentiable}\}.
\]
When $d=0$, the function $\xi$ is also required to be continuous, and we write $\mathcal{C}_n(T)$ instead of $\mathcal{C}^0_n(T)$.

\subsection{Signal Temporal Logic} 

We use Signal Temporal Logic (STL) formulae~\cite{maler2004monitoring} to specify high-level temporal requirements for robot trajectories.
The syntax of STL formulae is given by
\begin{equation}\label{eq:generalSTL}
    \Phi ::=   \top\mid \pi^\mu \mid \neg \Phi \mid \Phi_1 \wedge \Phi_2 \mid \Phi_1 \mathbf{U}_{[a,b]} \Phi_2,
\end{equation}
where $\top$ denotes the \textsf{true} predicate, and $\pi^\mu$ is an atomic predicate associated with a predicate function $\mu:\mathbb{R}^n \to \mathbb{R}$.
The symbols $\neg$ and $\wedge$ denote negation and conjunction, respectively.
The operator $\mathbf{U}_{[a,b]}$ is the temporal ``\emph{until}'' operator, where $a,b\in \mathbb{R}_{\geq 0}$ and $a \leq b$.
The \emph{horizon} of an STL formula $\Phi$, denoted by $\mathcal{H}(\Phi)$, is defined recursively by
(i) $\mathcal{H}(\pi^\mu)=0$, 
(ii) $\mathcal{H}(\neg \Phi)=\mathcal{H}(\Phi)$, 
(iii) $\mathcal{H}(\Phi_1\wedge \Phi_2)
=
\max\{ \mathcal{H}(\Phi_1),\mathcal{H}(\Phi_2) \}$, and 
(iv)
$\mathcal{H}(\Phi_1 \mathbf{U}_{[a,b]} \Phi_2)
=
b+\max\{ \mathcal{H}(\Phi_1),\mathcal{H}(\Phi_2) \}$.
In this work, we consider STL formulae $\Phi$ satisfying $\mathcal{H}(\Phi)\leq T$.

Throughout the paper, the robot configuration space is $\mathbb{R}^n$.
Given a trajectory $\xi \in \mathcal{C}_n(T)$ and an STL formula $\Phi$, we write $(\xi,t) \models \Phi$ if $\xi$ satisfies $\Phi$ at time $t$.
The semantics of STL formulae are defined recursively as follows:
\begin{align*}
& \left(\xi, t\right) \vDash \pi^{\mu} 
&& \!\!\!\!\!\!\Leftrightarrow  
\mu\left(\xi(t)\right) \geq 0,\\
&\left(\xi, t\right) \vDash \neg \Phi
&&\!\!\!\!\!\!\Leftrightarrow  
 \left(\xi, t\right) \nvDash \Phi,
\\
&\left(\xi, t\right) \vDash \Phi_1 \wedge \Phi_2 
&&\!\!\!\!\!\!\Leftrightarrow  
\left(\xi, t\right) \vDash \Phi_1
\text{ and }
\left(\xi, t\right) \vDash \Phi_2,
\\
&\left(\xi, t\right) \vDash \Phi_1 \mathbf{U}_{[a, b]} \Phi_2
&&\!\!\!\!\!\!\Leftrightarrow  \!\!\!\!\!\!\!\!
\begin{array}{l l}
     &  \exists t^{\prime} \in\left[t+a, t+b\right]:\left(\xi, t^{\prime}\right) \vDash \Phi_2 \\
& 
\text{and }
\forall t^{\prime \prime} \in\left[t, t^{\prime}\right]:
\left(\xi, t^{\prime \prime}\right) \vDash \Phi_1 . 
\end{array}
\end{align*}
The reader is referred to~\cite{maler2004monitoring} for further details on STL semantics.
We also use the following derived temporal operators:
\begin{itemize}
	\item 
	``\emph{eventually}''
	$\mathbf{F}_{[a,b]} \Phi:= \top \mathbf{U}_{[a,b]} \Phi$.
	Thus, $(\xi,t) \models \mathbf{F}_{[a,b]} \Phi$ holds if   $(\xi,t') \models \Phi$ for some $t'\!\in\! [t+a,t+b]$;
	
	\item 
	``\emph{always}''
	$\mathbf{G}_{[a,b]} \Phi:=\neg \mathbf{F}_{[a,b]} \neg \Phi$.
	Thus, $(\xi,t) \models \mathbf{G}_{[a,b]} \Phi$ holds if $(\xi,t') \models \Phi$ for all $t'\!\in\! [t+a,t+b]$.
\end{itemize}

For simplicity, we write $\xi \models \Phi$ when $(\xi,0) \models \Phi$.
For an atomic predicate $\pi^\mu$, its satisfaction region is defined as
$[\pi^\mu]=\{ x \in \mathbb{R}^n \mid \mu(x)\geq 0 \}$. 
In particular, $[\neg \pi^\mu] = \mathbb{R}^n\setminus [\pi^\mu]$.
Moreover, for atomic predicates $\varphi_1$ and $\varphi_2$, we have $[\varphi_1\wedge\varphi_2]=[\varphi_1]\cap [\varphi_2]$.

\subsection{B\'ezier Curves}

In this work, we use B\'ezier curves to parameterize robot trajectories.
Given $K \in \mathbb{Z}_{\geq 0}$, the Bernstein polynomials of degree $K$ are defined by
\begin{align}
\label{eq:berst}
\beta_k^{K} (s) = \binom{K}{k}
s^k
(1 - s)^{K-k},\quad s \in [0,1],
\end{align}
for each $k \in \{0,\dots,K\}$.
By the binomial theorem, $\sum_{k=0}^K\beta_k^K(s)=1$.
Therefore, for any $s \in [0,1]$, the scalars $\beta_k^K(s)$ form the coefficients of a convex combination.

Given control points $\gamma_0, \ldots, \gamma_K \in \mathbb{R}^n$, the corresponding B\'ezier curve $\Gamma: [0,1] \to \mathbb{R}^n$ is the vector-valued polynomial
\begin{align}
\label{eq:bez}
\Gamma(s) = \sum_{k=0}^K \beta_k^K (s) \gamma_k.
\end{align}
We recall several standard properties of B\'ezier curves that will be used later.
The reader is referred to~\cite{farouki1988algorithms} for further details.

\begin{enumerate}
   \item \textbf{Derivative:}
   The derivative $\dot \Gamma$ of a B\'ezier curve $\Gamma$ is a B\'ezier curve of degree $K-1$ whose control points satisfy
   \begin{equation}\label{eq:bezierderi}
       \dot \gamma_k = K (\gamma_{k+1} - \gamma_k),
\quad k \in \{ 0,\dots, K-1\}.
   \end{equation}
We denote by $\Gamma^{(i)}$ and $\gamma^{(i)}$   the $i$-th derivative of the B\'ezier curve and its  control points, respectively.

\item \textbf{Endpoint:}
The B\'ezier curve starts at its first control point and ends at its last control point, i.e., 
\begin{equation}\label{eq:beziercontrol}
    \Gamma(0) = \gamma_0\quad  \text{ and }\quad 
\Gamma(1) = \gamma_K.
\end{equation}

\item \textbf{Convex hull:}
The B\'ezier curve remains in the convex hull of its control points, i.e.,
\begin{equation}\label{eq:bezierconvex}
    \Gamma(s) \in \operatorname{conv} (\{\gamma_0, \ldots, \gamma_K\}),
\quad \forall s \in [0,1].
\end{equation}
\end{enumerate}

\subsection{Graphs of Convex Sets}

A graph of convex sets (GCS) is described by the tuple
\[
\mathcal{G} = (\mathcal{V}, \mathcal{E}, \mathbb{V}=\{\mathcal{X}_v\}_{v\in \mathcal{V}}, \mathbb{E}=\{\mathcal{X}_e\}_{e \in \mathcal{E}},l_e),
\]
where $(\mathcal{V}, \mathcal{E})$ is a directed graph with vertex set $\mathcal{V}$ and edge set $\mathcal{E} \subset \mathcal{V} \times \mathcal{V}$.
Each vertex $v\in\mathcal{V}$ is associated with a convex set $\mathcal{X}_v$, and each edge $e=(u,w)\in\mathcal{E}$ is associated with a convex set $\mathcal{X}_e \subseteq \mathcal{X}_u \times \mathcal{X}_w$. 
The function $l_e:\mathcal{X}_e\to \mathbb{R}_{\geq 0}$ is a convex nonnegative length function associated with edge $e$.

A path $\mathbf{p}$ is a sequence of distinct vertices in $\mathcal{V}$.
We denote by $\mathcal{E}_{\mathbf{p}}$ the set of edges traversed by $\mathbf{p}$.
Given a GCS, the associated optimization problem is to find a shortest path from a source vertex $s \in \mathcal{V}$ to a target vertex $t \in \mathcal{V}$.
This problem can be written as
\begin{subequations}\label{eq:gcs}
\begin{align}
    \min ~& \sum_{e=(u,v)\in\mathcal{E}_\mathbf{p}} l_e(x_u, x_v) \label{eq:cost} \\
    \mathrm{s.t.}~& \mathbf{p} \in \mathcal{P}_{s,t}, \\
                  & x_v \in \mathcal{X}_v, \quad \forall v \in \mathbf{p}, \\
                  & (x_u, x_v) \in \mathcal{X}_e, \quad \forall e = (u,v) \in \mathcal{E}_\mathbf{p}.
\end{align}
\end{subequations}
Here, $\mathcal{P}_{s,t}$ denotes the set of paths from $s$ to $t$, and $x_v \in \mathcal{X}_v$ is the continuous variable associated with vertex $v$.

Problem~\eqref{eq:gcs} can be encoded exactly as a mixed-integer convex program (MICP)~\cite{marcucci2024shortest}.
Although solving the resulting MICP is NP-hard in general, it often admits a tight convex relaxation.
In many instances, a globally optimal solution to~\eqref{eq:gcs} can be obtained by solving the convex relaxation followed by a low-cost rounding step on the integer variables~\cite{marcucci2023motion}.
The reader is referred to~\cite{marcucci2023motion,marcucci2024shortest} for more details on GCS.
In this work, we use the GCS framework to formulate the STL motion-planning problem in a computationally efficient manner.

\section{Problem Formulation} \label{sec:problemformulation}

We consider the STL motion-planning problem for a robot operating under temporal task requirements and motion constraints.
The objective is to compute a trajectory that satisfies a given STL specification while respecting a prescribed velocity bound.
Although we do not explicitly impose the full system dynamics, we require the trajectory to satisfy a smoothness condition, which supports dynamically feasible execution for differentially flat systems~\cite{murray1995differential}.
The problem is formally stated as follows.

\begin{myprob}\label{prob:main}
Given an initial state $x_0 \in \mathbb{R}^n$, a time horizon $T\in \mathbb{R}_{\geq 0}$, an STL task $\Phi$ in \eqref{eq:generalSTL} satisfying $\mathcal{H}(\Phi)\leq T$, a convex velocity bound $\textsf{vel} \subseteq \mathbb{R}^n$,   a smoothness order $d \in \mathbb{Z}_{\geq 0}$,
and   a convex cost functional $J$, 
find a trajectory $\xi:[0,T]\to \mathbb{R}^n$ that solves the following optimization problem:
\begin{subequations}\label{eq:main_problem}
    \begin{align}
        \min_{\xi} ~& J(\xi) \\
        \mathrm{s.t.} ~& \xi \vDash \Phi, \label{eq:STLconsinopt}\\
                     & \xi \in \mathcal{C}^d_n(T),  \label{eq:smoothconstraint} \\
                     & \xi(0) = x_0, \label{eq:initialcons} \\
                     & \dot{\xi}(t) \in \textsf{vel}, \quad \forall t \in [0,T].\label{eq:velocityconstraint}
    \end{align}
\end{subequations} 
\end{myprob}

\section{STL Motion Planning via GCS}\label{sec:GCSforSTL}
This section presents the proposed procedure for converting the STL motion-planning problem in Problem~\ref{prob:main} into a shortest-path optimization problem over a graph of convex sets.
The overall approach consists of the following steps:
\begin{itemize} 
    \item First, we represent the STL specification by a timed automaton (TA), whose accepting runs provide a timed symbolic description of task satisfaction.
    
    \item Second, we combine the TA with a convex decomposition of the configuration space to construct a joint transition system (JTS), which simultaneously records task progress and the convex region occupied by the trajectory.
    
    \item Third, based on the JTS, we construct an instance of GCS in which vertices encode B\'ezier-curve segments, timing variables, and clock values, while edges encode continuity, smoothness, velocity, and timing constraints.
    
    \item Finally, we solve the resulting GCS shortest-path problem, either exactly as a mixed-integer convex program or through its convex relaxation followed by rounding, and reconstruct a continuous-time B\'ezier-spline trajectory from the solution.
\end{itemize}
 We also establish soundness of the formulation and analyze its computational complexity and completeness-related limitations. 
Hereafter, Section~\ref{subsection:defofTA} introduces timed automata and joint transition systems.
Section~\ref{subsection:construct GCS} constructs the corresponding GCS.
Section~\ref{subsection:GCStheoreticalresult} provides the theoretical analysis.

\subsection{Timed Automata and Joint Transition System}\label{subsection:defofTA}

We first introduce timed automata and joint transition systems, which will be used to encode the temporal progress of an STL task and its coupling with the geometric decomposition of the configuration space.

We begin with the notion of a time constraint.
Let $c$ be a clock variable.
A time constraint on $c$, denoted by $\theta(c) \subseteq [0,T]$, is a closed interval of the form $a \leq c \leq b$.
The set of all time constraints on $c$ is denoted by $\Theta(c)$.
Given a set of clock variables $C=\{c_1,\dots,c_n\}$, we define
\[
\Theta(C)=\Theta(c_1)\times\cdots\times\Theta(c_n).
\]
That is, each $\theta\in \Theta(C)$ assigns a time constraint to every clock variable in $C$.
We say that a clock variable has \emph{no constraint} if $\theta(c)=[0,T]$.

\begin{mydef}[Timed Automaton]
A timed automaton (TA) is defined as the $7$-tuple
\begin{equation}\label{eq:TAdef}
    A=(S,S_0,R, C,\delta,M,S_F),
\end{equation}
where 
 $S$ is the set of states;
  $S_0\subseteq S$ is the set of initial states;
 $R=\{ r_1,\dots,r_k \}$ is a set of regions, where $r_i \subseteq \mathbb{R}^n$ for each $i=1,2,\dots,k$;
  $C$ is the set of clock variables;
  $\delta: S\times S \to \Theta(C)\times2^C$ is a partial transition function defined on all valid transitions of the TA.
    For $\delta(s,s')=(\theta,\gamma)$, $\theta$ is the time constraint associated with the transition, and $\gamma$ is the set of clock variables reset by the transition.
    We also write this transition as $s\xrightarrow{\theta,\gamma}s'$;
  $M:S\to R$ maps each TA state to a region such that, whenever the TA is in state $s\in S$, the trajectory must remain in region $M(s)$;
 $S_F \subseteq S$ is the set of accepting states. 
\end{mydef}

Before defining a run of a TA, we introduce several clock-related operations.
Let $\sigma: C \to [0,T]$ be a clock valuation, and let $\theta \in \Theta(C)$ be a time constraint.
We say that $\sigma$ satisfies $\theta$, denoted by $\sigma \models \theta$, if $\sigma(c)\in\theta(c)$ for every $c \in C$.
Given a clock valuation $\sigma$ and a time duration $t \in \mathbb{R}_{\geq 0}$, let $\sigma+t$ denote the clock valuation defined by
\[
(\sigma+t)(c)=\sigma(c)+t,\quad \forall c\in C.
\]
For a set of clock variables $\gamma \subseteq C$, we denote by $\sigma[\gamma]$ the clock valuation obtained by resetting the variables in $\gamma$, i.e.,
\[
\sigma[\gamma](c)=
\begin{cases}
0, & c\in \gamma,\\
\sigma(c), & c\notin \gamma.
\end{cases}
\]

\begin{mydef}[Run of a Timed Automaton]\label{def:runofTA}
Given a timed automaton $A=(S,S_0,R, C,\delta,M,S_F)$, a run $\rho$ of $A$ is a finite sequence
\begin{equation}\label{eq:timeautorun}
        \xrightarrow[\sigma_{0}]{} (s_0,I_0)
        \xrightarrow[\sigma_1]{\theta_1,\gamma_1} (s_1,I_1)
        \xrightarrow[\sigma_2]{\theta_2,\gamma_2}
        \cdots
        \xrightarrow[\sigma_{n}]{\theta_n,\gamma_n} (s_n,I_n)
\end{equation}
such that $s_0 \in S_0$, $I_0=[t_0=0,t_1]$ with $t_0\leq t_1$, $\sigma_0(c)=0$ for all $c \in C$, and for each $i \in \{1,2,\dots,n\}$, the following conditions hold:
\begin{enumerate}
    \item $\delta(s_{i-1},s_i)=(\theta_i,\gamma_i)$;
    \item $I_i=[t_i,t_{i+1}]$ with $t_i \leq t_{i+1}$ and $t_{n+1}=T$;
    \item $(\sigma_{i-1}+t_i-t_{i-1}) \models \theta_i$;
    \item $\sigma_i=(\sigma_{i-1}+t_i-t_{i-1})[\gamma_i]$.
\end{enumerate}
\end{mydef}

The definitions of timed automata and their runs are similar to those in~\cite{alur1996benefits}.
We next provide an intuitive interpretation of a TA run.
Suppose that the run stays in state $s$ over the time interval $I=[t,t']$, and that the clock valuation at time $t$ is $\sigma$.
Then each clock variable evolves continuously over $I$; specifically, its value at time $\tau \in I$ is $\sigma(c)+\tau-t$ for every $c\in C$.
When the transition $s\xrightarrow{\theta,\gamma}s'$ occurs at time $t'$, the updated clock valuation before reset must satisfy the transition constraint, i.e., $(\sigma+t'-t)(c)\in \theta(c),\forall c\in C$.
After the transition, the run enters state $s'$, and the new clock valuation becomes $\sigma'=(\sigma+t'-t)[\gamma]$. 
This procedure is then repeated along the run.
We say that a run $\rho$ of the timed automaton $A$ in \eqref{eq:timeautorun} is \emph{accepting}, denoted by $\rho \models A$, if $s_n \in S_F$, i.e., if the run ends in an accepting state.

We now define the trajectories induced by TA runs.

\begin{mydef}[Trajectory Induced by a Run]\label{def:runinducetra}
Given a run $\rho$ of a timed automaton $A$ in \eqref{eq:timeautorun}, a trajectory $\xi\in \mathcal{C}_n(T)$ is said to be induced by $\rho$ if
\begin{equation}\label{eq:inducedtrajectory}
    \xi(\tau) \in M(s_i), \quad \forall \tau \in I_i,\ i \in \{0,1,\dots,n \}.
\end{equation}
A trajectory $\xi$ is said to be accepting with respect to $A$, denoted by $\xi \models A$, if there exists a run $\rho$ in \eqref{eq:timeautorun} such that $\rho \models A$ and $\rho$ induces $\xi$.
\end{mydef}

We say that a TA $A$ in \eqref{eq:TAdef} satisfies an STL formula $\Phi$ in \eqref{eq:generalSTL}, denoted by $A \models \Phi$, if, for every trajectory $\xi\in \mathcal{C}_n(T)$, it holds that
\begin{equation}\label{eq:desired}
    \xi \models A \implies \xi \models \Phi.
\end{equation}
That is, every trajectory accepted by $A$ is guaranteed to satisfy the STL formula $\Phi$.

We now define the \emph{joint transition system} used to combine the TA with the geometry of the configuration space.
Suppose that we are given a convex decomposition
\[
\mathbb{D}=\{D_1,D_2,\dots,D_k\},
\]
where $D_i \subseteq \mathbb{R}^n$ for each $i=1,2,\dots,k$.
Such a decomposition can be constructed manually for low-dimensional configuration spaces.
For high-dimensional robotic systems, one may instead use algorithms such as IRIS~\cite{deits2015computing,werner2024faster} and C-IRIS~\cite{dai2024certified} to compute convex regions efficiently.
The joint transition system combines the symbolic progress of the TA with the geometric adjacency information of the convex decomposition.

\begin{mydef}[Joint Transition System]
Given a convex decomposition $\mathbb{D}$ and a TA $A=(S,S_0,R,C,\delta,M,S_F)$, the joint transition system (JTS) $\mathbb{T}$ is defined as the $6$-tuple
\begin{equation}\label{eq:transitionsystemdef}
     \mathbb{T}=(Q,Q_0,Q_F, C, \delta^i,\delta^o),
\end{equation}
where
\begin{itemize}
     \item $Q\subseteq S\times \mathbb{D}$ is the set of states such that $(s,D)\in Q$ if $D \subseteq M(s)$;
     \item $Q_0 \subseteq Q$ is the set of initial states such that $(s,D) \in Q_0$ if $s \in S_0$;
     \item $Q_F\subseteq Q$ is the set of accepting states such that $(s,D) \in Q_F$ if $s \in S_F$;
     \item $C$ is the set of clock variables;
     \item $\delta^i \subseteq Q \times Q$ is the set of inner transitions.
     Specifically, for $q=(s,D)$ and $q'=(s',D')$, we have $(q,q')\in \delta^i$ if $s=s'$, $D\neq D'$, and $D\cap D'\neq \emptyset$;
     \item $\delta^o : Q \times Q \to \Theta(C) \times 2^C$ is a partial function for outer transitions.
     Specifically, for $q=(s,D)$ and $q'=(s',D')$, we have $\delta^o(q,q')=(\theta,\gamma)$ if $\delta(s,s')=(\theta,\gamma)$ and $D\cap D'\neq \emptyset$.
\end{itemize}
\end{mydef}

For each JTS state $q=(s,D)\in Q$, the component $s$ records the task progress in the TA, while the component $D$ records the convex region occupied by the trajectory.
In particular, the task progress must start from an initial TA state $s_0\in S_0$.
An inner transition keeps the TA state unchanged and allows the trajectory to move between adjacent convex regions.
By contrast, an outer transition advances the TA state from $s$ to $s'$ while also moving the trajectory between intersecting convex regions.
Finally, once the trajectory reaches a JTS state $q=(s,D)\in Q_F$, the corresponding TA state satisfies $s\in S_F$, and hence the STL task is certified by the accepting condition of the TA.

\subsection{Construction of GCS}\label{subsection:construct GCS}

Given a joint transition system
$\mathbb{T}=(Q,Q_0,Q_F,C,\delta^i,\delta^o)$ in \eqref{eq:transitionsystemdef},
an initial state $x_0\in\mathbb{R}^n$,
a final time $T\in\mathbb{R}_{\geq 0}$,
a convex velocity bound $\textsf{vel}\subseteq\mathbb{R}^n$,
a smoothness order $d\in\mathbb{Z}_{\geq 0}$,
and a B\'ezier degree $K\in\mathbb{Z}_{\geq 0}$ with $K\geq d$,
we construct a graph of convex sets (GCS)
\begin{equation}\label{eq:GCSfromTS}
    \mathcal{G}
    =
    (\mathcal{V}, \mathcal{E},
    \mathbb{V}=\{\mathcal{X}_v\}_{v\in \mathcal{V}},
    \mathbb{E}=\{\mathcal{X}_e\}_{e\in \mathcal{E}},
    l_e),
\end{equation}
whose components are defined as follows.

\textbf{(1) Vertex set:}
The vertex set is defined as
\[
\mathcal{V}=Q\cup\{s,t\},
\]
where $s$ and $t$ are virtual source and target vertices, respectively.
Each vertex $q\in Q$ corresponds to a state of the joint transition system and therefore records both the progress of the timed automaton and the convex region occupied by the trajectory.
The virtual source and target vertices are introduced to unify the treatment of multiple initial and accepting JTS states.

\textbf{(2) Edge set:}
The edge set is defined as
\[
\mathcal{E}
=
\mathcal{E}^i
\cup
\mathcal{E}^o
\cup
\mathcal{E}^s
\cup
\mathcal{E}^t,
\]
where $\mathcal{E}^i=\delta^i$, 
$\mathcal{E}^o=\delta^o$, 
and
$\mathcal{E}^s=\{s\}\times Q_0, 
\mathcal{E}^t=Q_F\times\{t\}$.
Here, $\mathcal{E}^i$ and $\mathcal{E}^o$ correspond to the inner and outer transitions of the JTS, respectively.
An inner edge allows the trajectory to move between adjacent convex regions while keeping the same TA state.
An outer edge advances the TA state and enforces the corresponding clock constraints.
The edge sets $\mathcal{E}^s$ and $\mathcal{E}^t$ connect the virtual source and target vertices to the initial and accepting JTS states.

\textbf{(3) Associated convex sets for vertices:}
For each JTS state $q=(s,D)\in Q$, the associated convex set satisfies
\begin{equation}\label{eq:associatedconvexset}
    \mathcal{X}_q
    \subseteq
    D^{K+1}\times [0,T]^{K+1}\times[0,T]^{|C|}.
\end{equation}
A point $x_q\in\mathcal{X}_q$ takes the form
\begin{equation}\label{eq:continuous pointinassociatedconvexset}
    x_q
    =
    (\nu_0,\nu_1,\dots,\nu_K,
    t_0,t_1,\dots,t_K,
    c_1,c_2,\dots,c_{|C|}),
\end{equation}
where
\begin{itemize} 
    \item $\nu_0,\nu_1,\dots,\nu_K$ are the control points of a B\'ezier curve $\Gamma_q^r$ that describes the geometric path segment;
    \item $t_0,t_1,\dots,t_K$ are the control points of a scalar B\'ezier curve $\Gamma_q^h$ that describes the time parameterization of this segment;
    \item $c_1,c_2,\dots,c_{|C|}$ record the clock values when the trajectory starts visiting the JTS state $q$.
\end{itemize}
Given $x_q\in\mathcal{X}_q$, the trajectory segment associated with $q$ is reconstructed as $\xi_q=\Gamma_q^r\circ(\Gamma_q^h)^{-1}$,
where $\xi_q:[t_0,t_K]\to D$.
To ensure that $\Gamma_q^h$ is invertible, we impose
\begin{equation}\label{eq:invertiblecons}
   t_{k+1}-t_k>0,
   \quad k=0,1,\dots,K-1.
\end{equation}
In numerical implementation, this strict inequality can be replaced by
$t_{k+1}-t_k\geq \varepsilon$ for a small $\varepsilon>0$.

To enforce the velocity bound in \eqref{eq:velocityconstraint}, we further impose
\begin{equation}\label{eq:velocityconsforgcs}
    \frac{\nu_{k+1}-\nu_k}{t_{k+1}-t_k}
    \in \textsf{vel},
    \quad k=0,1,\dots,K-1.
\end{equation}
Since $\textsf{vel}$ is convex and the time increments are positive, these constraints define a convex feasible set.
Therefore, $\mathcal{X}_q$ is characterized by
\eqref{eq:associatedconvexset}, \eqref{eq:invertiblecons}, and \eqref{eq:velocityconsforgcs}.

\textbf{(4) Associated convex constraint sets for edges:}
For each edge $e=(q,q')\in\mathcal{E}^i\cup\mathcal{E}^o$, where $q=(s,D)$ and $q'=(s',D')$, we define an associated convex constraint set
\[
\mathcal{X}_e\subseteq \mathcal{X}_q\times\mathcal{X}_{q'}.
\]
Let $x_q
=
(\nu_0,\nu_1,\dots,\nu_K,
t_0,t_1,\dots,t_K,
c_1,c_2,\dots,c_{|C|})$
and $x_{q'}
=
(\nu'_0,\nu'_1,\dots,\nu'_K,
t'_0,t'_1,\dots,t'_K,
c'_1,c'_2,\dots,c'_{|C|})$. 
The set $\mathcal{X}_e$ imposes two classes of constraints: (i) smoothness constraints for concatenating adjacent B\'ezier segments and 
(ii) timing constraints for enforcing the clock evolution of the TA.
\begin{itemize}
    \item 
    First, to ensure that the concatenated trajectory is $d$-times continuously differentiable, we impose
\begin{equation}\label{eq:differentiablecons}
    \nu_{K-m}^{(m)} = \nu_{0}^{\prime(m)},  
    t_{K-m}^{(m)} = t_{0}^{\prime(m)},
      m=0,1,\dots,d,
\end{equation}
where $\nu^{(m)}$ and $t^{(m)}$ denote the control points of the $m$-th derivatives of the corresponding B\'ezier curves.
When $m=0$, these constraints enforce continuity of the geometric path and the time parameterization.
When $m\geq 1$, they enforce matching derivatives up to order $d$.
    \item 
Second, we impose clock-evolution constraints.
The form of these constraints depends on whether $e$ is an inner edge or an outer edge.
Let the clock variables in $C$ be indexed as
$C=\{\chi_1,\dots,\chi_{|C|}\}$, and let $c_i$ and $c_i'$ denote the values of clock $\chi_i$ at the beginning of vertices $q$ and $q'$, respectively.
If $(q,q')\in\mathcal{E}^i$, then the TA state does not change.
Therefore, all clocks increase by the duration of the current segment:
\begin{equation}\label{eq:timecongcsinner}
    c_i+t_K-t_0=c'_i,
    \quad i=1,2,\dots,|C|.
\end{equation}
These are affine equality constraints.
If $(q,q')\in\mathcal{E}^o$, then the edge corresponds to an outer transition of the TA.
Let $\delta^o(q,q')=(\theta,\gamma)$, 
where $\theta$ is the time constraint and $\gamma\subseteq C$ is the set of clocks reset by this transition.
For each clock, we impose
\begin{equation} \label{eq:timevarcongcsouter}
\left\{
		\begin{array}{cl}
			 c'_i =c_i+t_K-t_0, & \text{if } \chi_i \notin \gamma,\\
			 c'_i=0, & \text{if } \chi_i \in \gamma.
		\end{array}
\right.
\end{equation}
Moreover, the clock valuation immediately before the transition must satisfy the transition guard:
\begin{equation}\label{eq:timecosgcsouter}
 c_i+t_K-t_0 \in \theta(\chi_i),
 \quad i=1,2,\dots,|C|.
\end{equation}
Since each $\theta(\chi_i)$ is a closed interval, the constraints in \eqref{eq:timecosgcsouter} are convex.
\end{itemize}

\textbf{(5) Source and target constraints:}
The virtual source and target vertices do not carry continuous trajectory variables.
Instead, they impose boundary constraints on their adjacent JTS vertices.
\begin{itemize}
    \item 
    For an edge $e=(s,q)\in\mathcal{E}^s$ with $q\in Q_0$, we impose
\begin{equation}\label{eq:edgeconsforsource}
  \nu_0=x_0,
  \quad
  t_0=0,\quad
  c_i=0,\quad i=1,2,\dots,|C|.
\end{equation}
Thus, the trajectory starts from the prescribed initial state $x_0$ at time $0$, and all clock variables are initialized to zero.
    \item 
For an edge $e=(q,t)\in\mathcal{E}^t$ with $q\in Q_F$, we impose
\begin{equation}\label{eq:edgeconsfortarget}
  t_K=T.
\end{equation}
Hence, the reconstructed trajectory is defined over the entire time horizon $[0,T]$.
\end{itemize}

The use of virtual source and target vertices has two advantages.
First, it provides a unified way to handle multiple initial TA states, multiple convex regions containing the initial state, and multiple accepting JTS states.
Second, the vertex and edge constraints associated with the JTS are independent of the initial state.
Thus, when the initial state changes, only the constraints associated with the source edges need to be modified.

\textbf{(6) Length function and trajectory reconstruction:}
Finally, we define the edge length function
$l_e:\mathcal{X}_e\to\mathbb{R}_{\geq 0}$.
A common choice is the length of the B\'ezier control polygon:
\begin{equation}\label{eq:lengthfunction}
   l_e(x_q,x_{q'})
   =
   \sum_{k=0}^{K-1}
   \left\Vert \nu_{k+1}-\nu_k \right\Vert .
\end{equation}
If the $\ell_2$ norm is used, then \eqref{eq:lengthfunction} gives a convex upper approximation of the actual curve length, and the resulting convex programs are second-order cone programs (SOCPs)~\cite{marcucci2024shortest}.
If the $\ell_1$ norm is used, the approximation is generally looser but the resulting programs become linear programs (LPs), which are more suitable for large-scale instances.
Since derivatives of B\'ezier curves are also B\'ezier curves, derivative-dependent costs can be incorporated in the same manner.

After constructing the GCS in \eqref{eq:GCSfromTS}, we solve the corresponding shortest-path problem in \eqref{eq:gcs}.
Suppose that the solution returns a path
\[
\mathbf{p}=s q_1q_2\dots q_n t
\]
from the virtual source $s$ to the virtual target $t$.
For each $i=1,2,\dots,n$, the solution provides a point
\begin{equation}\label{eq:continuous pointfromsolution}
    x_{q_i}
    =
    (\nu_0^i,\nu_1^i,\dots,\nu_K^i,
    t_0^i,t_1^i,\dots,t_K^i,
    c_1^i,c_2^i,\dots,c_{|C|}^i).
\end{equation}
Let $\Gamma_{q_i}^r$ and $\Gamma_{q_i}^h$ be the B\'ezier curves with control points
$\nu_0^i,\nu_1^i,\dots,\nu_K^i$
and
$t_0^i,t_1^i,\dots,t_K^i$,
respectively.
The continuous-time B\'ezier-spline trajectory is reconstructed as
\begin{equation}\label{eq:reconstruction}
 \xi(t)
 =
 \Gamma_{q_i}^r\circ(\Gamma_{q_i}^h)^{-1}(t),
 \quad
 \text{if } t \in [t_0^i,t_K^i].
\end{equation}
This reconstructed trajectory is the candidate solution to the original STL motion-planning problem.

\subsection{Theoretical Analysis of the GCS}\label{subsection:GCStheoreticalresult}

We now analyze the theoretical properties of the GCS formulation constructed in Section~\ref{subsection:construct GCS}.
We first establish soundness: every feasible solution of the GCS shortest-path problem yields a trajectory that satisfies the original STL motion-planning problem.

\begin{mythm}\label{thm:soundness}
Let $T\in\mathbb{R}_{\geq0}$ be the time horizon.
Given an STL task $\Phi$ in \eqref{eq:generalSTL}, let $A_\Phi$ be a TA in \eqref{eq:TAdef} such that $A_\Phi\models \Phi$.
Let $\mathbb{D}$ be a convex decomposition of the robot configuration space, and let $\mathbb{T}$ be the JTS in \eqref{eq:transitionsystemdef} constructed from $A_\Phi$ and $\mathbb{D}$.
For an initial state $x_0\in\mathbb{R}^n$, a convex velocity bound $\textsf{vel}\subseteq\mathbb{R}^n$, a smoothness order $d\in\mathbb{Z}_{\geq0}$, and a B\'ezier degree $K\in\mathbb{Z}_{\geq0}$ with $K\geq d$, let $\mathcal{G}$ be the GCS in \eqref{eq:GCSfromTS}.
Suppose that the shortest-path problem \eqref{eq:gcs} associated with $\mathcal{G}$ is feasible, and let $\xi_{\mathcal{G}}:[0,T]\to\mathbb{R}^n$ be the trajectory reconstructed from \eqref{eq:reconstruction}.
Then $\xi_{\mathcal{G}}$ is a feasible solution to Problem~\ref{prob:main}.
\end{mythm}

\begin{proof}
   Suppose that a solution to the optimization problem \eqref{eq:gcs} yields the path $\mathbf{p}=s q_0q_1\dots q_l t$.
For each vertex $q_i=(s_i,D_i)$ with $i=0,1,\dots, l$, let the associated point $x_{q_i}$ be given by
\begin{equation}\label{eq:continuous pointfromsolutionprove}
    x_{q_i}=(\nu_0^i,\nu_1^i,\dots,\nu_{K}^i,t^i_0,t^i_1,\dots, t^i_K,c^i_1,c^i_2\dots,c^i_{|C|}).
\end{equation}

Let $\Gamma_{q_i}^r$ and $\Gamma_{q_i}^h$ be the B\'ezier curve segments with control points $\nu_0^i,\nu_1^i,\dots,\nu_{K}^i$ and $t^i_0,t^i_1,\dots, t^i_K$, respectively.
For each $i=0,1,\dots,l$, both $\Gamma_{q_i}^r$ and $\Gamma_{q_i}^h$ are infinitely differentiable over $[0,1]$.
Moreover, by \eqref{eq:bezierderi} and \eqref{eq:invertiblecons}, we have $\dot{\Gamma}_{q_i}^h(\tau)>0$ for all $\tau \in [0,1]$ and all $i=0,1,\dots,l$.
Therefore, by the inverse function theorem~\cite[Appendix A]{kass2011geometrical}, the inverse function $(\Gamma_{q_i}^h)^{-1}$ is well-defined and is also infinitely differentiable.
Furthermore, by \eqref{eq:beziercontrol}, $(\Gamma_{q_i}^h)^{-1}$ maps $[t_0^i,t_K^i]$ to $[0,1]$.
The constraints \eqref{eq:edgeconsforsource} and \eqref{eq:edgeconsfortarget} ensure that the B\'ezier spline starts at time $0$ and ends at time $T$.
Hence, the reconstruction in \eqref{eq:reconstruction} is well-defined.

From \eqref{eq:edgeconsforsource}, the first control point $\nu_0^0$ of the first B\'ezier curve segment satisfies $\nu_0^0=x_0$.
Then, by the endpoint property \eqref{eq:beziercontrol} of B\'ezier curves, we have $\xi_{\mathcal{G}}(0)=x_0$. Hence, the initial condition \eqref{eq:initialcons} is satisfied.

We next show that $\xi_{\mathcal{G}}$ also satisfies the smoothness requirement in \eqref{eq:smoothconstraint}.
Since $\Gamma_{q_i}^r$ and $(\Gamma_{q_i}^h)^{-1}(t)$ are both infinitely differentiable, their composition is also infinitely differentiable.
Therefore, each segment $\Gamma_{q_i}^r\circ (\Gamma_{q_i}^h)^{-1}$ already satisfies \eqref{eq:smoothconstraint}; that is, $\xi_{\mathcal{G}}(t)$ is infinitely differentiable over $[0,T]$ except at the concatenation instants $t_0^j$ for $j=1,2,\dots,l$.
It remains to show that
\begin{equation}\label{eq:concertationdifferent}
   \left( \Gamma_{q_{i-1}}^r\circ (\Gamma_{q_{i-1}}^h)^{-1}\right)^{(m)}(t_K^{i-1}) = \left( \Gamma_{q_{i}}^r\circ (\Gamma_{q_{i}}^h)^{-1}\right)^{(m)}(t_0^{i})
\end{equation}
holds for $i=1,2,\dots,l$ and $m=0,1,\dots,d$.
Let $f_i=\Gamma^r_{q_i}$, $g_i=\Gamma^h_{q_i}$, and $y_i=g_i^{-1}$.
By the derivative property \eqref{eq:bezierderi} and the constraint \eqref{eq:differentiablecons}, we have, for $i=1,2,\dots,l$ and $m=0,1,\dots,d$,
\begin{equation}\label{eq:soundnessmiddle1}
    f_{i-1}^{(m)}(t_K^{i-1}) = f_i^{(m)}(t_0^i), \quad g_{i-1}^{(m)}(t_K^{i-1}) = g_i^{(m)}(t_0^i).
\end{equation}
We next show that, for $i=1,2,\dots,l$ and $m=0,1,\dots,d$,
\begin{equation}\label{eq:induction}
    y_{i-1}^{(m)}(t_K^{i-1}) = y_i^{(m)}(t_0^i). 
\end{equation}
We prove it by induction on $m$.
Since $y_i=g_i^{-1}$, \eqref{eq:soundnessmiddle1} immediately implies that \eqref{eq:induction} holds for $m=0$.
Moreover, since $g\circ y(x)=x$, differentiating this identity $m$ times yields an expression for $y^{(m)}$.
In particular, for $m=1$, we have
$y^{(1)}_i=\frac{1}{g^{(1)}_i(y_i(x))}$. Therefore, \eqref{eq:induction} also holds for $m=1$, since it is already true for $m=0$ and \eqref{eq:soundnessmiddle1} holds.
Now suppose that \eqref{eq:induction} holds for $m=0,1,\dots,k$ with $k\geq 1$. For $1<n = k+1\leq d$, we have
\begin{equation}
\begin{aligned}
    y^{(n)}_i&(x)=-\frac{1}{g^{(1)}_i(y_i(x))} \times \\
    &\sum_{j=2}^n g^{(j)}(y_i(x))B_{n,j}(y^{(1)}_i(x),y^{(2)}_i(x),\dots,y^{(n-j+1)}_i(x)),
\end{aligned} \nonumber
\end{equation}
where $B_{n,j}$ is the Bell polynomials.
This shows that $y^{(n)}_i$ can be expressed in terms of $g_i,g^{(1)}_i,\dots,g^{(n)}_i,y_i,y^{(1)}_i,\dots,y^{(n-1)}_i$. Therefore, by the induction hypothesis and \eqref{eq:soundnessmiddle1}, \eqref{eq:induction} also holds for $n$. Hence, \eqref{eq:induction} is true.
Now consider $h_i=f_i\circ y_i$. Then
\[
h_i^{(n)}\!(x)\!=\!\!\!\sum_{j=1}^n f^{(j)}_i\!(y_i(x))B_{n,j}(y^{(1)}_i\!(x),y^{(2)}_i\!(x),\cdot\cdot,y^{(n-j+1)}_i\!(x))
\]
for $i=0,1,\dots,l$ and $n\in \mathbb{Z}_{\geq0}$, where $B_{n,j}$ is the Bell polynomial.
Thus, $h_i^{(n)}$ can be expressed in terms of $f_i,f^{(1)}_i,\dots,f^{(n)}_i,y_i,y^{(1)}_i,\dots,y^{(n)}_i$.
Therefore, \eqref{eq:soundnessmiddle1} and \eqref{eq:induction} imply that \eqref{eq:concertationdifferent} holds. Hence, $\xi_\mathcal{G}$ satisfies the smoothness constraint \eqref{eq:smoothconstraint}.

We now turn to the convex velocity constraint \eqref{eq:velocityconstraint}.
For $t \in [t_0^i,t_K^i]$ with $i=0,1,\dots,l$, we have
\[
\begin{aligned}
    \dot{\xi}_{\mathcal{G}}(t)=&\left(\Gamma_{q_i}^r\circ (\Gamma_{q_i}^h)^{-1}\right)^{(1)}(t)
   =\frac{\dot{\Gamma}_{q_i}^r(\tau)}{\dot{\Gamma}_{q_i}^h(\tau)}\\
  =& \frac{\sum_{k=0}^{K-1} \beta_k^{K-1} (\tau)(\nu^i_{k+1}-\nu^i_{k})}{\sum_{k=0}^{K-1} \beta_k^{K-1} (\tau)(t^i_{k+1}-t^i_{k})}=\sum_{k=0}^{K-1}w_k\frac{\nu^i_{k+1}-\nu^i_{k}}{t^i_{k+1}-t^i_{k}}
\end{aligned}
\]
where $\tau=(\Gamma_{q_i}^h)^{-1}(t)$ and $w_k=\frac{\beta_k^{K-1} (\tau)(t^i_{k+1}-t^i_{k})}{\sum_{m=0}^{K-1} \beta_m^{K-1} (\tau)(t^i_{m+1}-t^i_{m})}$.
Here, the second equality follows from the inverse function theorem, and the third equality follows from \eqref{eq:bez}.
Since $\beta_k^{K-1}(\tau)\geq 0$ for $k=0,1,\dots,K-1$ and $t^i_{m+1}-t_m^i>0$ by \eqref{eq:invertiblecons}, we have $w_k\geq 0$ for $k=0,1,\dots,K-1$.
Therefore, the coefficients $w_k$ form a convex combination, and hence
\[
\dot{\xi}_{\mathcal{G}}(t) \in \texttt{conv}(\frac{\nu^i_{1}-\nu^i_{0}}{t^i_{1}-t^i_{0}},\frac{\nu^i_{2}-\nu^i_{1}}{t^i_{2}-t^i_{1}},\dots, \frac{\nu^i_{K}-\nu^i_{K-1}}{t^i_{K}-t^i_{K-1}} ).
\]
Since \textsf{vel} is convex and \eqref{eq:velocityconsforgcs} holds, we conclude that $\dot{\xi}_{\mathcal{G}}(t) \in \textsf{vel}$ for all $t \in [0,T]$. That is, $\xi_{\mathcal{G}}$ satisfies \eqref{eq:velocityconstraint}.

Finally, we consider the STL constraint \eqref{eq:STLconsinopt}.
Let $s_i$ and $D_i$ denote the TA state and convex set associated with the vertex $q_i$, that is, $q_i=(s_i,D_i)$ for $i=0,1,\dots,l$.
Let $\hat{s}_0\hat{s}_1\dots \hat{s}_k$ be the reduced TA state sequence such that $\hat{s}_{m-1} \neq \hat{s}_{m}$ for $m=1,2,\dots,k$, and suppose that there exist indices $i[0]=0 < i[1]<\dots <i[k] \leq l=i[k+1]-1$ such that, for each $j=0,1,\dots,k$,
\begin{equation}\label{eq:statetoeliminatestate}
    s_i = \hat{s}_{j}, \ i[j] \leq i \leq i[j+1]-1.
\end{equation}
For $j = 1,2,\dots, k$, let $(\theta_j,\gamma_j)=\delta(\hat{s}_{j-1},\hat{s}_j)$, and define
\begin{equation}\label{eq:soundnessdef}
    \begin{aligned}
    &I_j=[t^{i[j]}_0, t^{i[j+1]}_0], \\
        &\sigma_j(c_m)=c_{m}^{i[j]}, \quad \forall m=1,2,\dots |C|
    \end{aligned}
\end{equation}
for $j=0,1,\dots,k$, where $t^{i[k+1]}_0=t_{K}^{l}$.
We now prove that the run $\rho$ of the form
\begin{equation}\label{eq:timeautoruningcs}
            \xrightarrow[\sigma_{0}]{} (\hat{s}_0,I_0)\xrightarrow[\sigma_1]{\theta_1,\gamma_1} (\hat{s}_1,I_1)\xrightarrow[\sigma_2]{\theta_2,\gamma_2}\cdots \xrightarrow[\sigma_{k}]{\theta_k,\gamma_k} (\hat{s}_k,I_k)
    \end{equation}
is well-defined in the TA $A_\Phi$.

By the definition of $\mathcal{E}^s$, we have $q_0 \in Q_0$, and thus $\hat{s}_0=s_0 \in S_0$, where $S_0$ is the set of initial states of the TA $A_\Phi$.
Moreover, by \eqref{eq:edgeconsfortarget} and \eqref{eq:edgeconsforsource}, we have $t_0^{i[k+1]}=t_K^l=T$ and $t_0^{i[0]}=0$, respectively.
In addition, \eqref{eq:differentiablecons} and \eqref{eq:invertiblecons} imply that $t_{0}^{i[j]} \leq t_{0}^{i[j+1]}$ for $j=0,1,\dots,k$.
Therefore, the interval $I_j$ in \eqref{eq:soundnessdef} is well-defined in the sense of Definition~\ref{def:runofTA}.
Finally, \eqref{eq:edgeconsforsource} implies that $\sigma_0(c_m)=c_m^{i[0]}=0$ for all $m=1,2,\dots,|C|$.

To prove that \eqref{eq:timeautoruningcs} is well-defined, it remains to show that \eqref{eq:soundnessdef} satisfies conditions 3) and 4) in Definition~\ref{def:runofTA} for $j=1,2,\dots,k$.
From \eqref{eq:statetoeliminatestate} we know that $s_i = \hat{s}_{j-1},  i[j-1] \leq i \leq i[j]-1$, i.e., $(q_i,q_{i+1}) \in \mathcal{E}^i$ for $i[j-1] \leq i \leq i[j]-2$.
Then from \eqref{eq:edgeconsforsource}, we have
\[
c_m^i+t_K^i-t_0^i=c^{i+1}_m,\forall m=1,2,\dots,|C|, i[j-1] \leq i \leq i[j]-2.
\]
By summing the above equation from $i=i[j-1]$ to $i=i[j]-2$ and using $t_K^{i[j]-2}=t_0^{i[j]-1}$ from \eqref{eq:timecongcsinner}, we obtain
\begin{equation}\label{eq:soundnessinner}
    c^{i[j-1]}_{m} + t_0^{i[j]-1} - t_0^{i[j-1]}=c_m^{i[j]-1}, \quad m=1,2,\dots,|C|.
\end{equation}
Since $s_{i[j]-1} \neq s_{i[j]}$, we have $(q_{i[j]-1},q_{i[j]}) \in \mathcal{E}^o$.
Then, by \eqref{eq:timevarcongcsouter}, we have for $m=1,2,\dots,|C|$,
\begin{equation}\label{eq:soundnessreset}
    \left\{
		\begin{array}{cl}
			 c_m^{i[j]} =c_m^{i[j]-1}+t_K^{i[j]-1}-t_0^{i[j]-1},& \text{ if } \chi_m \notin \gamma_{j} \\
			 c_m^{i[j]}=0, &\text{ if } \chi_m \in \gamma_j  
		\end{array}.
		\right.
\end{equation}
Therefore, for $\chi_m \notin \gamma_j$, we have
\begin{equation}\label{eq:soundesstimecont}
    \begin{aligned}
     \sigma_{j}&(c_m)=c_{m}^{i[j]}=c_m^{i[j]-1}+t_K^{i[j]-1}-t_0^{i[j]-1}
     \\
     &=c_m^{i[j-1]}+t_K^{i[j]-1}-t_0^{i[j-1]} = \sigma_{j-1}(c_m)+t_0^{i[j]}-t_0^{i[j-1]},
\end{aligned}
\end{equation}
where the first and last equalities follow from \eqref{eq:soundnessdef} and \eqref{eq:differentiablecons}, respectively, and the second and third equalities follow from \eqref{eq:soundnessreset} and \eqref{eq:soundnessinner}, respectively.
Moreover, \eqref{eq:soundnessreset} implies that $\sigma_{j}(c_m)=c_m^{i[j]}=0$ for $\chi_m \in \gamma_j$.
Combining this with \eqref{eq:soundesstimecont}, we conclude that condition 4) of Definition~\ref{def:runofTA} is satisfied for \eqref{eq:timeautoruningcs}.

Since $(q_{i[j]-1},q_{i[j]}) \in \mathcal{E}^o$, \eqref{eq:timecosgcsouter} implies that
\[
 c_m^{i[j]-1}+t_K^{i[j]-1}-t_0^{i[j]-1}\in \theta_m(c_m),\quad m=1,2,\dots,|C|.   
\]
Combining this with \eqref{eq:soundnessinner} and the identity $t_K^{i[j]-1}=t_0^{i[j]}$, we obtain
\begin{equation}
c^{i[j-1]}_{m} +t_0^{i[j]} - t_0^{i[j-1]} \in \theta_m(c_m),\quad m=1,2,\dots,|C|. \nonumber
\end{equation}
Therefore, condition 3) of Definition~\ref{def:runofTA} is satisfied for \eqref{eq:timeautoruningcs}.
Hence, \eqref{eq:timeautoruningcs} is a well-defined run in the TA $A_\Phi$.
Moreover, by the definition of $\mathcal{E}^t$, we have $\hat{s}_k \in S_F$. That is, the run $\rho$ in \eqref{eq:timeautoruningcs} is an accepting run of the TA $A_\Phi$.

We now show that $\xi_\mathcal{G}$ is induced by the run $\rho$. For each $i=0,1,\dots,l$, the control points of the B\'ezier curve segment $\Gamma_{q_i}^r$ lie in $D_i$. Hence, by \eqref{eq:bezierconvex}, we have $\Gamma_{q_i}^r(\tau) \in D_i$ for all $\tau \in [0,1]$.
Moreover, by the definition of the JTS in \eqref{eq:transitionsystemdef}, we have $D_i \subseteq M(s_i)$ for every $i=0,1,\dots,l$.
Combining this with \eqref{eq:reconstruction}, we obtain
\[
\xi_\mathcal{G}(t) \in D_i \subseteq M(s_i), \quad \forall  t \in [t_0^i, t_K^i],\ i=0,1,\dots,l.
\]
Combining this with \eqref{eq:statetoeliminatestate}, we get
\[
\xi_\mathcal{G}(t) \in M(\hat{s}_j), \quad \forall  t \in [t_0^{i[j]}, t_0^{i[j+1]}],\ j=0,1,\dots,k.
\]
Since $\rho$ in \eqref{eq:timeautoruningcs} is accepting, \eqref{eq:inducedtrajectory} implies that $\xi_{\mathcal{G}} \models A_\Phi$.
Therefore, $\xi_{\mathcal{G}} \models \Phi$, which completes the proof.
\end{proof}

Theorem~\ref{thm:soundness} shows that the proposed GCS formulation is sound: any feasible solution of the GCS shortest-path problem yields a continuous-time B\'ezier-spline trajectory satisfying the STL task, the prescribed smoothness requirement, the initial condition, and the velocity bound.

We next analyze the computational complexity of the proposed formulation when the GCS shortest-path problem is solved via convex relaxation and rounding, as in~\cite{marcucci2023motion}.
In practice, as shown in Section~\ref{sec:simulation} and in~\cite{marcucci2023motion,kurtz2023temporal}, the gap between the convex relaxation and the original MICP is often small.
Moreover, in all our examples, only a few rounding iterations are sufficient to recover a feasible trajectory.
Therefore, we focus on the complexity of solving the convex relaxation.

\begin{mypro}\label{prop:complexity}
Given the GCS $\mathcal{G}$ in \eqref{eq:GCSfromTS}, the complexity of approximately solving the convex relaxation of its shortest-path problem is polynomial in the number of JTS states $|Q|$, the number of clock variables $|C|$, the configuration-space dimension $n$, and the B\'ezier degree $K$.
Equivalently, since $Q\subseteq S\times\mathbb{D}$, this complexity is polynomial in $|S|$, $|\mathbb{D}|$, $|C|$, $n$, and $K$.
\end{mypro}

\begin{proof}
As discussed in \eqref{eq:lengthfunction}, the convex relaxation is an LP when the $\ell_1$ norm is used and an SOCP when the $\ell_2$ norm is used.
By standard interior-point complexity results~\cite{nesterov1994interior}, LPs and SOCPs can be approximately solved in time polynomial in the number of decision variables and constraints.
It therefore suffices to count how these quantities scale.

For each JTS state $q\in Q$, the associated variable $x_q$ contains $(K+1)$ geometric control points in $\mathbb{R}^n$, $(K+1)$ timing control points, and $|C|$ clock values.
Hence, the total number of decision variables is
$|Q|\big((n+1)(K+1)+|C|\big)$.
The vertex constraints include the monotonicity constraints in \eqref{eq:invertiblecons} and the velocity constraints in \eqref{eq:velocityconsforgcs}, whose number scales polynomially with $|Q|$, $K$, and $n$.
For each edge in $\mathcal{E}^i\cup\mathcal{E}^o$, the smoothness constraints in \eqref{eq:differentiablecons} contribute a number of affine constraints polynomial in $K$, $d$, and $n$, while the clock-evolution and guard constraints in \eqref{eq:timecongcsinner}--\eqref{eq:timecosgcsouter} contribute a number of constraints polynomial in $|C|$.
Thus, the total number of constraints is polynomial in
\[
|Q|,\quad |\mathcal{E}^i|+|\mathcal{E}^o|,\quad |C|,\quad n,\quad K.
\]
Since $|\mathcal{E}^i|+|\mathcal{E}^o|\leq |Q|^2$, the number of constraints is polynomial in $|Q|$, $|C|$, $n$, and $K$.
Therefore, the convex relaxation can be approximately solved in polynomial time with respect to these quantities.
Finally, because $Q\subseteq S\times\mathbb{D}$, we have $|Q|\leq |S||\mathbb{D}|$.
This proves the claim.
\end{proof}

Proposition~\ref{prop:complexity} shows that, once the TA and the convex decomposition are fixed, the resulting convex optimization problem scales polynomially with the configuration-space dimension and the B\'ezier degree.
This property is particularly important for high-dimensional robotic systems, where direct time-discretized mixed-integer STL encodings often become computationally expensive.

\begin{remark}
A desirable property of a planning method is completeness, namely, that the method returns a solution whenever Problem~\ref{prob:main} is feasible.
However, the proposed approach is not complete in this full sense as the following cases may occur. 
First, a feasible trajectory $\xi$ satisfying the STL formula $\Phi$  may not be accepted by the chosen TA $A_\Phi$, i.e., $\xi\models\Phi$ while $\xi\not\models A_\Phi$.
Second, a feasible trajectory may leave the constructed convex decomposition $\mathbb{D}$.
Also,  a feasible trajectory may not admit a B\'ezier-spline representation with the prescribed degree $K$ and the prescribed segmentation. 
Therefore, the proposed GCS optimization is complete only with respect to the restricted class of trajectories that are accepted by the TA, lie inside the chosen convex decomposition, and admit the prescribed B\'ezier-spline parameterization satisfying the imposed control-point constraints.
Moreover, when the solution is obtained through convex relaxation and rounding, the rounding procedure may fail to recover a feasible integer path even if one exists.
Nevertheless, as shown in the experiments, this relaxation-and-rounding strategy is effective in practice.
\end{remark} 
 
\section{Efficient Timed Automata Construction for an Expressive STL Fragment}
\label{sec:timeautomaton}

In principle, a timed automaton can be constructed for a general STL formula by using standard STL-to-TA translations, such as that in~\cite{alur1996benefits}.
However, directly using such general-purpose constructions in the proposed GCS framework is challenging for two reasons.
First, standard constructions are designed to handle arbitrary nesting of temporal operators.
As a result, the automaton must retain sufficient timing information to support future temporal compositions, which can lead to a very large state space even for seemingly simple specifications.
Second, the size of the resulting automaton may depend not only on the syntactic size of the formula, but also on the temporal constants appearing in the formula.
This dependence is particularly undesirable for long-horizon robotic tasks, where large time intervals are common.

To improve scalability, this section identifies a restricted but still expressive STL fragment that covers many robotic motion-planning tasks, including reachability, invariance, reach-and-stay, recurrence, and their Boolean combinations.
For each temporal pattern in this fragment, we directly provide a compact TA template.
General formulae in the fragment are then handled by composing these templates through union and product constructions, which correspond to disjunction and conjunction, respectively.
In this way, compositional construction is only needed for Boolean operators, while temporal operators are handled directly by templates.

\subsection{Considered STL Fragment}
\label{subsection:STLfragment}

We consider the following fragment of STL formulae:
\begin{subequations} \label{eq:consideredstl}
	\begin{align}
    \varphi
    &::= \pi^\mu
    \mid \neg \pi^\mu
    \mid \varphi_1 \wedge \varphi_2,
    \label{eq:firstclassstl}  \\
	\phi
    &::= \mathbf{F}_{[a,b]} \varphi
    \mid \mathbf{G}_{[a,b]} \varphi
    \mid \varphi_1\mathbf{U}_{[a,b]}\varphi_2
    \nonumber\\
    &\quad
    \mid \mathbf{F}_{[a,b]}\mathbf{G}_{[c,d]}\varphi
    \mid \mathbf{G}_{[a,b]}\mathbf{F}_{[c,d]}\varphi,
    \label{eq:secondclassstl} \\
	\Phi
    &::= \phi
    \mid \Phi_1 \wedge \Phi_2
    \mid \Phi_1 \vee \Phi_2.
    \label{eq:thirdclassstl}
	\end{align}
\end{subequations}
Here, $\varphi_1$ and $\varphi_2$ are formulae of class $\varphi$, while $\Phi_1$ and $\Phi_2$ are formulae of class $\Phi$.
Formulae in class $\varphi$ describe time-independent Boolean combinations of atomic predicates.
Formulae in class $\phi$ describe elementary temporal patterns applied to such Boolean predicates.
Finally, formulae in class $\Phi$ are obtained by taking conjunctions and disjunctions of elementary temporal specifications.

The fragment in \eqref{eq:consideredstl} is expressive enough to encode many common robotic tasks.
For example, $\mathbf{F}_{[a,b]}\varphi$ represents reaching a desired region within a time window;
$\mathbf{G}_{[a,b]}\varphi$ represents remaining in a safe region over a time window;
$\varphi_1\mathbf{U}_{[a,b]}\varphi_2$ represents maintaining one condition until another condition is reached;
$\mathbf{F}_{[a,b]}\mathbf{G}_{[c,d]}\varphi$ represents reaching a region and staying there for a prescribed duration; and
$\mathbf{G}_{[a,b]}\mathbf{F}_{[c,d]}\varphi$ represents recurrent satisfaction within a sliding time window.
However, this fragment does not include arbitrary nesting of temporal operators.
For instance, formulae such as
$\mathbf{G}_{[a,b]}\big(\pi^{\mu_1}\wedge
\mathbf{F}_{[c,d]}\pi^{\mu_2}\big)$ 
are not directly included in \eqref{eq:consideredstl}.
This restriction is intentional: it allows each temporal pattern in \eqref{eq:secondclassstl} to be handled by a small dedicated TA template, thereby avoiding the state-space explosion caused by repeatedly composing temporal operators.

\subsection{Timed Automata Templates and Boolean Composition}
\label{subsection:construction of TA}

In this subsection, we first introduce compact TA templates for the elementary temporal formulae $\phi$ in \eqref{eq:secondclassstl}.
We then define two Boolean composition operators, namely union and product, which are used to construct TAs for formulae in \eqref{eq:thirdclassstl}.

We begin with several common conventions used throughout this subsection.
Let $\Omega=\mathbb{R}^n$ denote the whole configuration space.
For a Boolean formula $\varphi$ in \eqref{eq:firstclassstl}, let $[\varphi]$ denote its satisfaction region.
Since the subsequent GCS construction requires convex regions, the regions used in the TA templates may be conservative subsets of satisfaction regions.
Specifically, we use $R_\varphi\subseteq[\varphi]$ to denote a region in which $\varphi$ is guaranteed to hold.
For conjunctions, we similarly use
$R_{\varphi_1\wedge\varphi_2}\subseteq[\varphi_1\wedge\varphi_2]$.
For the template of $\mathbf{G}_{[a,b]}\mathbf{F}_{[c,d]}\varphi$, we additionally use a region $R_{\neg\varphi}$ satisfying $[\neg\varphi]\subseteq R_{\neg\varphi}$, which represents a region where $\varphi$ is not enforced.

We use $\top_C$ to denote the unconstrained guard over the clock set $C$.
The clock $\kappa$ records the global physical time and is never reset.
For nested temporal templates, an auxiliary clock $\eta$ is introduced to measure the duration associated with the inner temporal requirement.
All states not explicitly specified as accepting are non-accepting.

The elementary TA templates are summarized as follows. 

\textbf{1) Eventually:} $\phi=\mathbf{F}_{[a,b]}\varphi$. \quad 
Let $A_\phi=(S,S_0,R,C,\delta,$ $M,S_F)$, where
$S=\{s_0,s_1,s_2\}$, $S_0=\{s_0\}$, $S_F=\{s_1,s_2\}$,
$R=\{\Omega,R_\varphi\}$, and $C=\{\kappa\}$.
The region map is $M(s_0)=M(s_2)=\Omega$ and $M(s_1)=R_\varphi$.
The transitions are
\[
\delta(s_0,s_1)=(a\leq\kappa\leq b,\emptyset),
\qquad
\delta(s_1,s_2)=(\top_C,\emptyset).
\]

\textbf{2) Always:} $\phi=\mathbf{G}_{[a,b]}\varphi$.\quad 
Let $A_\phi=(S,S_0,R,C,\delta,M,$ $S_F)$, where
$S=\{s_0,s_1,s_2\}$, $S_0=\{s_0\}$, $S_F=\{s_2\}$,
$R=\{\Omega,R_\varphi\}$, and $C=\{\kappa\}$.
The region map is $M(s_0)=M(s_2)=\Omega$ and $M(s_1)=R_\varphi$.
The transitions are
\[
\delta(s_0,s_1)=(\kappa\leq a,\emptyset),
\qquad
\delta(s_1,s_2)=(\kappa\geq b,\emptyset).
\]

 \textbf{3) Until:} $\phi=\varphi_1\mathbf{U}_{[a,b]}\varphi_2$. \quad
Let $A_\phi=(S,S_0,R,C,\delta,$ $M,S_F)$, where
$S=\{s_0,s_1,s_2\}$, $S_0=\{s_0\}$, $S_F=\{s_1,s_2\}$,
$C=\{\kappa\}$, and
$R=\{\Omega,R_{\varphi_1},R_{\varphi_1\wedge\varphi_2}\}$.
The region map is
$M(s_0)=R_{\varphi_1}$, $M(s_1)=R_{\varphi_1\wedge\varphi_2}$, and $M(s_2)=\Omega$.
The transitions are
\[
\delta(s_0,s_1)=(a\leq\kappa\leq b,\emptyset),
\qquad
\delta(s_1,s_2)=(\top_C,\emptyset).
\]

\textbf{4) Eventually-always:} $\phi\!=\!\mathbf{F}_{[a,b]}\mathbf{G}_{[c,d]}\varphi$. \quad
Let $A_\phi\!=(S,S_0,R,$ $C,\delta,M,S_F)$, where
$S=\{s_0,s_1,s_2\}$, $S_0=\{s_0\}$, $S_F=\{s_2\}$,
$R=\{\Omega,R_\varphi\}$, and $C=\{\kappa,\eta\}$.
The region map is $M(s_0)\!=\!M(s_2)\!=\!\Omega$ and $M(s_1)\!=\!R_\varphi$.
The transitions are
\[
\delta(s_0,s_1)\!=\!(a+c\!\leq\!\kappa\!\leq\! b+c,\{\eta\}), \ 
\delta(s_1,s_2)\!=\!(\eta\geq d-c,\emptyset).
\]

\textbf{5) Always-eventually:} $\phi=\mathbf{G}_{[a,b]}\mathbf{F}_{[c,d]}\varphi$.\quad
Let $A_\phi=(S,S_0,R,$ $C,\delta,M,S_F)$, where
$S=\{s_0,s_1,s_2,s_3,s_4\}$, $S_0=\{s_0\}$, $S_F=\{s_4\}$,
$R=\{\Omega,R_\varphi,R_{\neg\varphi}\}$, and $C=\{\kappa,\eta\}$.
The region map is
$M(s_0)=M(s_1)=M(s_4)=\Omega$, 
$M(s_2)=R_\varphi$, and
$M(s_3)=R_{\neg\varphi}$. 
The transitions are
\[
\begin{aligned}
&\delta(s_0,s_1)=(\kappa\leq a+c,\{\eta\}),\quad
\delta(s_1,s_2)=(\eta\leq d-c,\emptyset),\\
&\delta(s_2,s_3)=(\top_C,\{\eta\}),\quad
\delta(s_3,s_2)=(\eta\leq d-c,\emptyset),\\
& \delta(s_2,s_4)=(\kappa\geq b+c,\emptyset).
\end{aligned}
\]

The intuition behind these templates is as follows.
For $\mathbf{F}_{[a,b]}\varphi$, the automaton enters a $\varphi$-enforcing state at some time within $[a,b]$.
For $\mathbf{G}_{[a,b]}\varphi$, the automaton enters the $\varphi$-enforcing state no later than $a$ and leaves it no earlier than $b$.
For $\varphi_1\mathbf{U}_{[a,b]}\varphi_2$, the state $s_0$ enforces $\varphi_1$ until a switching time in $[a,b]$, and the state $s_1$ certifies that $\varphi_2$ also holds at that switching time.
For $\mathbf{F}_{[a,b]}\mathbf{G}_{[c,d]}\varphi$, the first transition selects the start of an interval on which $\varphi$ must hold, while the auxiliary clock $\eta$ enforces the required dwell time $d-c$.
For $\mathbf{G}_{[a,b]}\mathbf{F}_{[c,d]}\varphi$, the automaton repeatedly returns to the $\varphi$-enforcing state $s_2$ within every duration of length $d-c$, ensuring that each window $[\tau+c,\tau+d]$ with $\tau\in[a,b]$ contains a time instant satisfying $\varphi$.
Therefore, each elementary temporal formula in \eqref{eq:secondclassstl} is represented by a constant-size TA.
Only one clock is needed for non-nested temporal operators, and two clocks are sufficient for the nested templates $\mathbf{F}\mathbf{G}$ and $\mathbf{G}\mathbf{F}$.
The templates are illustrated in Figs.~\ref{fig:automatonGanduntil} and~\ref{fig:automatonFGandGF}.

\begin{figure}[t]
	 \subfigure[Case of $\varphi_1\mathbf{U}_{[a,b]}\varphi_2$.] 
	{\label{fig:automatonuntil}
 \begin{minipage}[b]{0.46\linewidth}
	\centering
 \includegraphics[height=0.8cm]{ 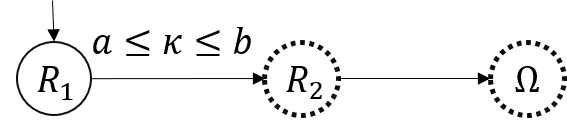}
    \end{minipage}
	}
    \subfigure[Case of $\mathbf{G}_{[a,b]}\varphi$.] 
	{\label{fig:automatonG}
 \begin{minipage}[b]{0.46\linewidth}
	\centering
 \includegraphics[height=0.8cm]{ 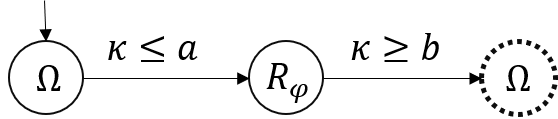}
    \end{minipage}
	}
    \caption{Timed automata for $\varphi_1\mathbf{U}_{[a,b]}\varphi_2$ and $\mathbf{G}_{[a,b]}\varphi$. Dashed states are accepting states. Here $R_1=R_{\varphi_1},R_2=R_{\varphi_1\wedge \varphi_2}$.}
   \label{fig:automatonGanduntil}
\end{figure}

\begin{figure}[t]
	 \subfigure[Case of $\mathbf{F}_{[a,b]}\mathbf{G}_{[c,d]}\varphi$.] 
	{\label{fig:automatonFG}
 \begin{minipage}[b]{0.4\linewidth}
	\centering
 \includegraphics[height=1.6cm]{ 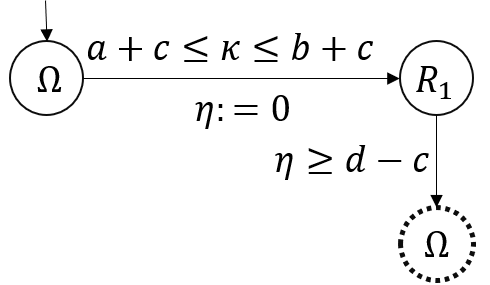}
    \end{minipage}
	}\subfigure[Case of $\mathbf{G}_{[a,b]}\mathbf{F}_{[c,d]}\varphi$.] 
	{\label{fig:automatonGF}
 \begin{minipage}[b]{0.56\linewidth}
	\centering
 \includegraphics[height=1.6cm]{ 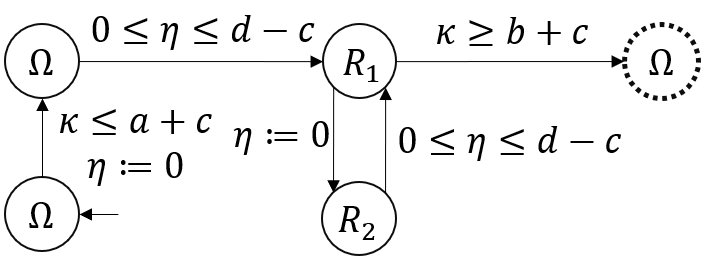}
    \end{minipage}
	}
    \caption{Timed automata for $\mathbf{F}_{[a,b]}\mathbf{G}_{[c,d]}\varphi$ and $\mathbf{G}_{[a,b]}\mathbf{F}_{[c,d]}\varphi$. Dashed states are accepting states. Here $R_1=R_\varphi$, $R_2=R_{\neg \varphi}$}
   \label{fig:automatonFGandGF}
\end{figure}

The templates above handle only the elementary temporal formulae in \eqref{eq:secondclassstl}.
To construct TAs for Boolean combinations in \eqref{eq:thirdclassstl}, we use two standard composition operators: union for disjunction and product for conjunction.
Before applying these operators, we rename states and clocks if necessary so that the automata to be composed have disjoint state sets and disjoint clock sets.

\begin{mydef}[Union of TAs]
\label{def:unionofTAs}
Given two TAs
$A_i=(S_i,S_{i,0},R_i,C_i,\delta_i,M_i,S_{i,F})$, $i=1,2$,
with $S_1\cap S_2=\emptyset$ and $C_1\cap C_2=\emptyset$,
their union $A_1\bigsqcup A_2$ is the TA
\[
A=(S_1\cup S_2, S_{1,0}\cup S_{2,0}, R_1\cup R_2, C_1\cup C_2, \delta, M, S_{1,F}\cup S_{2,F}),
\]
where, for $i=1,2$, $M(s)=M_i(s)$ for $s\in S_i$, and
$\delta(s,s')=\delta_i(s,s')$ whenever $s,s'\in S_i$ and $\delta_i(s,s')$ is defined.
No transition is added between $S_1$ and $S_2$.
\end{mydef}

\begin{mydef}[Product of TAs]
\label{def:productofTA}
Given two TAs
$A_i=(S_i,S_{i,0},R_i,C_i,\delta_i,M_i,S_{i,F})$, $i=1,2$,
with $S_1\cap S_2=\emptyset$ and $C_1\cap C_2=\emptyset$,
their product $A_1\bigotimes A_2$ is the TA
\[
A = (S_1\times S_2,  S_{1,0}\times S_{2,0},  R,  C_1\cup C_2, \delta,  M, S_{1,F}\times S_{2,F}),
\]
where $R=\{r_1\cap r_2\mid r_1\in R_1,\ r_2\in R_2\}$ and
$M((s_1,s_2))=M_1(s_1)\cap M_2(s_2)$.
The transition function is defined asynchronously as follows.
If $\delta_1(s_1,s_1')=(\theta_1,\gamma_1)$ and $s_2'=s_2$, then
\[
\delta((s_1,s_2),(s_1',s_2'))=(\theta,\gamma_1),
\]
where $\theta$ agrees with $\theta_1$ on $C_1$ and has no constraint on $C_2$.
Similarly, if $\delta_2(s_2,s_2')=(\theta_2,\gamma_2)$ and $s_1'=s_1$, then
\[
\delta((s_1,s_2),(s_1',s_2'))=(\theta,\gamma_2),
\]
where $\theta$ agrees with $\theta_2$ on $C_2$ and has no constraint on $C_1$.
\end{mydef}

The union construction naturally represents disjunction.
Since no transition is added between the two component automata, an accepting run of $A_1\bigsqcup A_2$ must stay entirely inside either $A_1$ or $A_2$.
Thus, a trajectory accepted by the union satisfies at least one of the two component specifications.

The product construction naturally represents conjunction.
A product state records the simultaneous progress of both component automata, and its region map is the intersection of the two component region maps.
Therefore, a trajectory induced by a product run satisfies the geometric constraints required by both automata.
The asynchronous transition rule allows either component automaton to advance whenever its own guard is satisfied, while the other component remains in its current state.
If two component transitions occur at the same physical time, they can be represented by two consecutive zero-duration transitions.
Consequently, the TA for a formula $\Phi$ in \eqref{eq:thirdclassstl} can be constructed recursively by
\[
A_{\Phi_1\vee\Phi_2}=A_{\Phi_1}\bigsqcup A_{\Phi_2},
\qquad
A_{\Phi_1\wedge\Phi_2}=A_{\Phi_1}\bigotimes A_{\Phi_2}.
\]

\subsection{Correctness and Complexity}
\label{subsection:TAproperty}

We now prove that the proposed construction is sound.
That is, every trajectory accepted by the constructed TA satisfies the corresponding STL formula.

\begin{mypro}
\label{prop:secondclassisright}
For any STL formula $\phi$ in \eqref{eq:secondclassstl}, let $A_\phi$ be the corresponding TA template constructed in Section~\ref{subsection:construction of TA}.
Then, for any trajectory $\xi:[0,T]\to\mathbb{R}^n$,
\begin{equation}
\label{eq:timaautomatoniscorrect}
    \xi \models A_\phi \implies \xi \models \phi.
\end{equation}
\end{mypro}

\begin{proof}
We prove the result by considering the five templates in \eqref{eq:secondclassstl}.

\textbf{1) Eventually:}
Let $\phi=\mathbf{F}_{[a,b]}\varphi$ and suppose that $\xi\models A_\phi$.
Then there exists an accepting run induced by $\xi$ that enters state $s_1$ at some time $t_1$ satisfying $a\leq t_1\leq b$.
Since $M(s_1)=R_\varphi\subseteq[\varphi]$, we have $(\xi,t_1)\models\varphi$.
Therefore, $\xi\models\mathbf{F}_{[a,b]}\varphi$.

\textbf{2) Always:}
Let $\phi=\mathbf{G}_{[a,b]}\varphi$ and suppose that $\xi\models A_\phi$.
By the guards of the template, the accepting run enters state $s_1$ at some time $t_1\leq a$ and leaves it at some time $t_2\geq b$.
Since $M(s_1)=R_\varphi\subseteq[\varphi]$, the trajectory satisfies $\varphi$ for all $t\in[t_1,t_2]$.
In particular, it satisfies $\varphi$ for all $t\in[a,b]$.
Hence $\xi\models\mathbf{G}_{[a,b]}\varphi$.

\textbf{3) Until:}
Let $\phi=\varphi_1\mathbf{U}_{[a,b]}\varphi_2$ and suppose that $\xi\models A_\phi$.
Then the accepting run stays in $s_0$ until some transition time $t_1$ satisfying $a\leq t_1\leq b$.
During this interval, the trajectory lies in
$R_{\varphi_1}\subseteq[\varphi_1]$.
At time $t_1$, the run enters $s_1$, whose region satisfies
$R_{\varphi_1\wedge\varphi_2}\subseteq[\varphi_1\wedge\varphi_2]$.
Therefore, $(\xi,t)\models\varphi_1$ for all $t\in[0,t_1]$, and $(\xi,t_1)\models\varphi_2$.
Thus, $\xi\models\varphi_1\mathbf{U}_{[a,b]}\varphi_2$.

\textbf{4) Eventually-always:}
Let $\phi=\mathbf{F}_{[a,b]}\mathbf{G}_{[c,d]}\varphi$ and suppose that $\xi\models A_\phi$.
By the first transition guard, the accepting run enters state $s_1$ at some time $t_1$ satisfying
\[
a+c\leq t_1\leq b+c.
\]
Let $\tau=t_1-c$.
Then $\tau\in[a,b]$.
Since the clock $\eta$ is reset when the run enters $s_1$ and the transition from $s_1$ to $s_2$ requires $\eta\geq d-c$, the run remains in $s_1$ for at least $d-c$ time units.
Hence the trajectory lies in $R_\varphi\subseteq[\varphi]$ for all times in
\[
[t_1,t_1+d-c]=[\tau+c,\tau+d].
\]
Therefore, $(\xi,\tau)\models\mathbf{G}_{[c,d]}\varphi$ for some $\tau\in[a,b]$, and consequently
\[
\xi\models\mathbf{F}_{[a,b]}\mathbf{G}_{[c,d]}\varphi.
\]

\textbf{5) Always-eventually:}
Let $\phi=\mathbf{G}_{[a,b]}\mathbf{F}_{[c,d]}\varphi$ and suppose that $\xi\models A_\phi$.
The run reaches the $\varphi$-enforcing state $s_2$ for the first time no later than $a+d$.
Indeed, the transition from $s_0$ to $s_1$ occurs no later than $a+c$, and the transition from $s_1$ to $s_2$ occurs within at most $d-c$ additional time units.
Afterwards, whenever the run leaves $s_2$, the clock $\eta$ is reset, and the transition from $s_3$ back to $s_2$ must occur within at most $d-c$ time units.
Thus, over the relevant interval, the time gap between consecutive visits to $s_2$ is never larger than $d-c$.
Moreover, the run can leave $s_2$ for the accepting state $s_4$ only after time $b+c$.

We now show that every window $[\tau+c,\tau+d]$ with $\tau\in[a,b]$ contains a time instant at which $\varphi$ holds.
Assume, for contradiction, that there exists $\tau\in[a,b]$ such that
\[
(\xi,t)\not\models\varphi,\quad \forall t\in[\tau+c,\tau+d].
\]
Then the interval $[\tau+c,\tau+d]$ contains no time at which the run is in state $s_2$, because $M(s_2)=R_\varphi\subseteq[\varphi]$.
However, the length of this interval is $d-c$, while the construction above ensures that every time interval of length $d-c$ between the first possible visit to $s_2$ and the terminal guard contains a visit to $s_2$.
This is a contradiction.
Therefore, for every $\tau\in[a,b]$, there exists $t\in[\tau+c,\tau+d]$ such that $(\xi,t)\models\varphi$.
Hence
\[
\xi\models\mathbf{G}_{[a,b]}\mathbf{F}_{[c,d]}\varphi.
\]
This proves the claim for all templates.
\end{proof}

Proposition~\ref{prop:secondclassisright} shows that each elementary template is sound.
If the conservative inclusions such as $R_\varphi\subseteq[\varphi]$ are replaced by equalities, then the templates for
$\mathbf{F}_{[a,b]}\varphi$,
$\mathbf{G}_{[a,b]}\varphi$,
$\varphi_1\mathbf{U}_{[a,b]}\varphi_2$, and
$\mathbf{F}_{[a,b]}\mathbf{G}_{[c,d]}\varphi$
are also complete with respect to their corresponding STL formulae.
That is, every trajectory satisfying the STL formula is accepted by the associated TA.
For the template of $\mathbf{G}_{[a,b]}\mathbf{F}_{[c,d]}\varphi$, this converse implication does not necessarily hold.
The reason is that the automaton explicitly alternates between a $\varphi$-enforcing state and a non-enforcing state.
If the regions are chosen exactly as $R_\varphi=[\varphi]$ and $R_{\neg\varphi}=[\neg\varphi]$, then a direct transition between these two states may require the trajectory to pass through
$[\varphi]\cap[\neg\varphi]$, which is empty.
Thus, this template is intended as a sound sufficient construction.

We next show that the union and product constructions correctly encode disjunction and conjunction.

\begin{mypro}
\label{prop:productandunion}
Let $A_i=(S_i,S_{i,0},R_i,C_i,\delta_i,M_i,S_{i,F}),i=1,2$ 
be two TAs with disjoint state sets and disjoint clock sets.
Let $\Phi_1$ and $\Phi_2$ be two STL formulae.
Suppose that, for $i=1,2$,
\[
\xi\models A_i \implies \xi\models\Phi_i.
\]
Then
\begin{enumerate}[leftmargin=*,topsep=0pt,itemsep=2pt]
    \item $\xi\models A_1\bigsqcup A_2
    \implies
    \xi\models\Phi_1\vee\Phi_2$;
    \item $\xi\models A_1\bigotimes A_2
    \implies
    \xi\models\Phi_1\wedge\Phi_2$.
\end{enumerate}
Moreover, if the converse implication also holds for each $A_i$, then the converse implication also holds for the union and product constructions.
\end{mypro}
The proof of Propositions~\ref{prop:productandunion} is provided in Appendix.
Combining Propositions~\ref{prop:secondclassisright} and~\ref{prop:productandunion}, we obtain a sound TA construction for every formula $\Phi$ in the fragment \eqref{eq:consideredstl}.
Specifically, we first build a constant-size template for each elementary temporal subformula $\phi$ in \eqref{eq:secondclassstl}.
We then recursively apply union and product operators according to the Boolean structure of $\Phi$.
The resulting TA $A_\Phi$ satisfies
\[
\xi\models A_\Phi \implies \xi\models\Phi.
\]

\begin{remark}[Complexity of the Proposed Construction]
\label{remark:complexityofTAconstruction}
We briefly analyze the complexity of the proposed TA construction.
For an STL formula $\Phi$ in \eqref{eq:thirdclassstl}, let $|\Phi|$ denote the number of elementary temporal subformulae $\phi$ of the form in \eqref{eq:secondclassstl} contained in $\Phi$.
Each elementary template introduced above has at most five states and uses at most two clock variables.
Therefore, before any optional clock-sharing simplification, the number of clocks in the composed TA satisfies $|C|\leq 2|\Phi|$, 
and hence grows linearly with $|\Phi|$.

The number of states depends on how these elementary templates are combined.
The union operator, used for disjunction, adds the state spaces of the component automata.
In contrast, the product operator, used for conjunction, multiplies the state spaces of the component automata.
Therefore, the state space may grow exponentially with the number of elementary temporal subformulae.
In the worst case, if $N=|\Phi|$ elementary templates are combined through conjunctions, the resulting product automaton has at most $|S|\leq 5^N$
states.
Thus, the proposed construction has linear clock growth and, in the worst case, exponential state growth with respect to the number of elementary temporal subformulae.
This exponential dependence is mainly caused by Boolean conjunctions, rather than by the temporal templates themselves.
\end{remark}

\begin{remark}[Comparison with Standard STL-to-TA Constructions]
\label{remark:comparisonstandardTA}
We now compare the proposed template-based construction with standard STL-to-TA translations.
For a general STL formula $\Phi$ in \eqref{eq:generalSTL}, let $N$ denote the total number of propositions, Boolean operators, and temporal operators in $\Phi$, and let $K$ denote the largest integer constant appearing in $\Phi$.
Classical constructions, such as~\cite{alur1996benefits}, may yield timed automata whose number of states is exponential in $NK$, while the number of clock variables is linear in $NK$.
This is because such constructions are designed for arbitrary STL formulae and must support recursive composition with temporal operators.
Although our construction still has worst-case exponential state growth due to Boolean products, it has two practical advantages for the fragment \eqref{eq:consideredstl}.
First, each elementary temporal pattern in \eqref{eq:secondclassstl} is represented by a constant-size template.
Therefore, the construction avoids recursively composing automata with temporal operators.
Second, the size of each template does not depend on the numerical lengths of the temporal intervals.
Thus, long time horizons or large interval bounds do not by themselves increase the size of the elementary automata.

This distinction is important for robotic motion-planning tasks.
For example, consider
\[
\Phi=
\mathbf{G}_{[a,b]}\mathbf{F}_{[c,d]}
\left(
\bigwedge_{i=1}^{n}
\bigvee_{j=1}^{m}
\varphi_{i,j}
\right).
\]
Under the proposed fragment-level counting, the outer temporal structure corresponds to a single elementary temporal subformula, so $|\Phi|=1$.
The Boolean formula inside the predicate only affects the geometric satisfaction region used in the corresponding template.
By contrast, a general syntactic count includes all propositions and Boolean operators inside the predicate, which can be much larger.
Therefore, the proposed construction is particularly suitable for long-horizon robotic tasks whose temporal structure is simple but whose geometric predicates may be rich.
\end{remark}

\begin{remark}[Implementation Simplifications]
Several implementation-level simplifications can further reduce the size of the constructed automaton.
First, the physical-time clock $\kappa$ used in each elementary template records the same global time and is never reset.
Thus, instead of assigning a separate physical-time clock to each template, all templates can share a single global clock.
With this simplification, the final automaton only needs one global clock, together with one auxiliary clock for each nested temporal template.
Second, unreachable or inconsistent product states can be pruned.
For example, if the clock constraints associated with two component states are mutually incompatible, then the corresponding product state or transition cannot appear in any feasible run and can be safely removed.
Such pruning can substantially reduce the size of the product automaton before constructing the GCS.
Third, when the full product automaton is still too large, one may select a particular accepting path, or a small set of promising accepting paths, from the TA and construct the GCS only over the selected paths.
This heuristic reduces the size of the resulting optimization problem but sacrifices completeness.
Nevertheless, soundness is preserved: whenever the resulting GCS problem is feasible, the reconstructed trajectory is still guaranteed to satisfy the STL specification.
\end{remark}
 
\section{Experiment Results}
\label{sec:simulation}

This section evaluates the proposed STL motion-planning framework in both simulation and hardware experiments.
The experiments are designed to examine three aspects of the proposed approach: its performance on standard low-dimensional STL planning benchmarks, its scalability to high-dimensional robotic systems, and its practicality on a real robot platform.
Section~\ref{subsec:compareexp} first reports simulation experiments on $2$-D and $3$-D STL planning problems, including comparisons with existing STL motion-planning methods.
Section~\ref{subsection:highdimexp} then evaluates the proposed method on a $30$-DoF humanoid robot in simulation.
Finally, Section~\ref{subsection:hardwardexp} presents a hardware experiment on a UR-3 robot arm.

All experiments are performed on a Linux workstation.
We use \textsf{Mosek}~\cite{mosek} to solve the convex relaxation of the GCS optimization problem.
The reported runtime of our method is divided into three parts: the time for constructing the TA from the STL specification, the time for forming the GCS, and the time for solving the resulting convex optimization problem.
Among these steps, only the final optimization step needs to be performed online.
The TA construction and GCS construction can be performed offline and reused for different initial conditions whenever the task structure and workspace decomposition remain unchanged.

The runtime results are summarized in Table~\ref{tab:runningtime}.
For the $2$-D benchmarks, the table also reports the solve times of the comparison methods.
In addition, for the $3$-D quadrotor example, we compare our method with the PWL method in~\cite{sun2022multi}.
A timeout, denoted by ``TO'', means that the solve time exceeds $7200$ seconds.

\begin{table*}[t]
  \centering
  \caption{Runtime comparison and implementation statistics. ``TO'' denotes timeout after $7200$ seconds, and ``--'' means that the corresponding method was not evaluated for that example. PWL and MPC correspond to the methods in~\cite{sun2022multi} and~\cite{kurtz2022mixed}, respectively.}
  \label{tab:runningtime}
  \begin{threeparttable}[t]
     \begin{tabular}{lcccccc}
      \toprule
      Example & Setting & STL-to-TA (s) & Form GCS (s) & Ours Solve (s) & PWL Solve (s) & MPC Solve (s) \\
      \midrule
      \texttt{stlcg}     & $2$-D   & $0.0003$ & $0.03$ & $0.39$ & $0.3$   & $0.2$ \\
      \texttt{puzzle-1}  & $2$-D   & $0.037$  & $1.23$ & $4.5$  & $18.2$  & $3017$ \\
      \texttt{puzzle-2}  & $2$-D   & $0.23$   & $4.35$ & $33.1$ & $90.9$  & TO \\
      \texttt{rover}     & $2$-D   & $0.001$  & $0.35$ & $3.86$ & $12.1$  & $866.4$ \\
      \texttt{either-or} & $2$-D   & $0.002$  & $1.43$ & $1.86$ & $0.3$   & TO \\
      \texttt{deliver}   & $2$-D   & $0.004$  & $0.33$ & $2.76$ & $8.2$   & TO \\
      \texttt{quadrotor} & $3$-D   & $0.004$  & $1.17$ & $15.68$ & $489$  & -- \\
      \texttt{humanoid}  & $30$-DoF & $0.001$ & $0.38$ & $10.2$ & -- & -- \\
      \texttt{manipulator} & $6$-DoF & $0.002$ & $0.46$ & $9.65$ & -- & -- \\
      \bottomrule
    \end{tabular}
  \end{threeparttable}
\end{table*}

\begin{figure*}[t]
    \centering
    \def\twoDCaseHeight{0.192\textheight}
    \newcommand{\twoDLegendLine}[1]{%
        \tikz[line cap=round]{\draw[#1,line width=0.9pt] (0,0)--(0,0.56cm);}%
    }
    \newcommand{\twoDLegendName}[1]{\rotatebox{90}{\scriptsize #1}}
    \begin{tabular}{@{}c@{\hspace{0.006\textwidth}}c@{}}
        \begin{minipage}[c]{0.052\textwidth}
            \centering
            \twoDLegendLine{brown!85!black}\par\vspace{0.05cm}
            \twoDLegendName{GCS}\par\vspace{0.34cm}
            \twoDLegendLine{purple,dashed}\par\vspace{0.05cm}
            \twoDLegendName{PWL}\par\vspace{0.34cm}
            \twoDLegendLine{green!55!black,dash pattern=on 4pt off 1.2pt on 0.9pt off 1.2pt}\par\vspace{0.05cm}
            \twoDLegendName{stlpy\_MICP}
        \end{minipage}
        &
        \begin{minipage}[c]{0.928\textwidth}
            \centering
            \begin{tabular}{@{}c@{\hspace{0.001\linewidth}}c@{\hspace{0.004\linewidth}}c@{}}
            \subfigure[stlcg]{\includegraphics[height=\twoDCaseHeight]{ 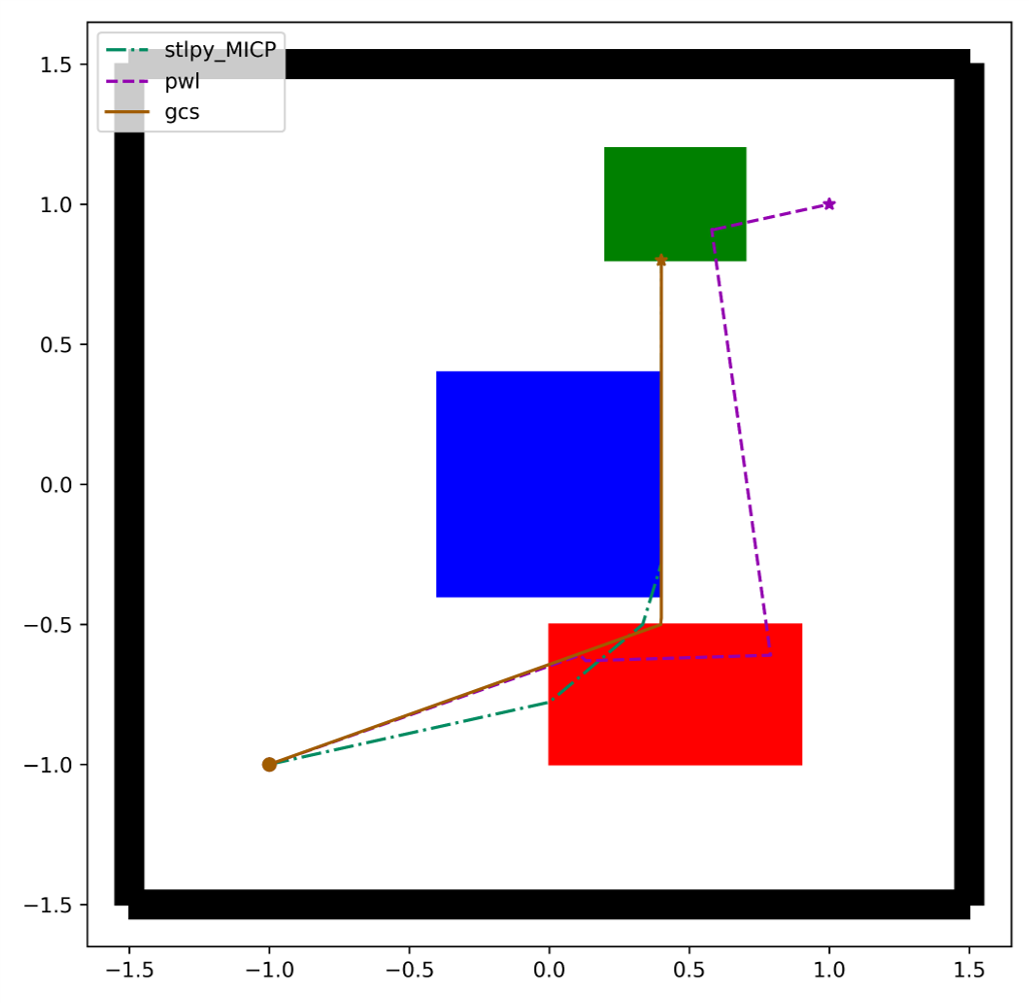}} &
            \subfigure[puzzle-1]{\includegraphics[height=\twoDCaseHeight]{ 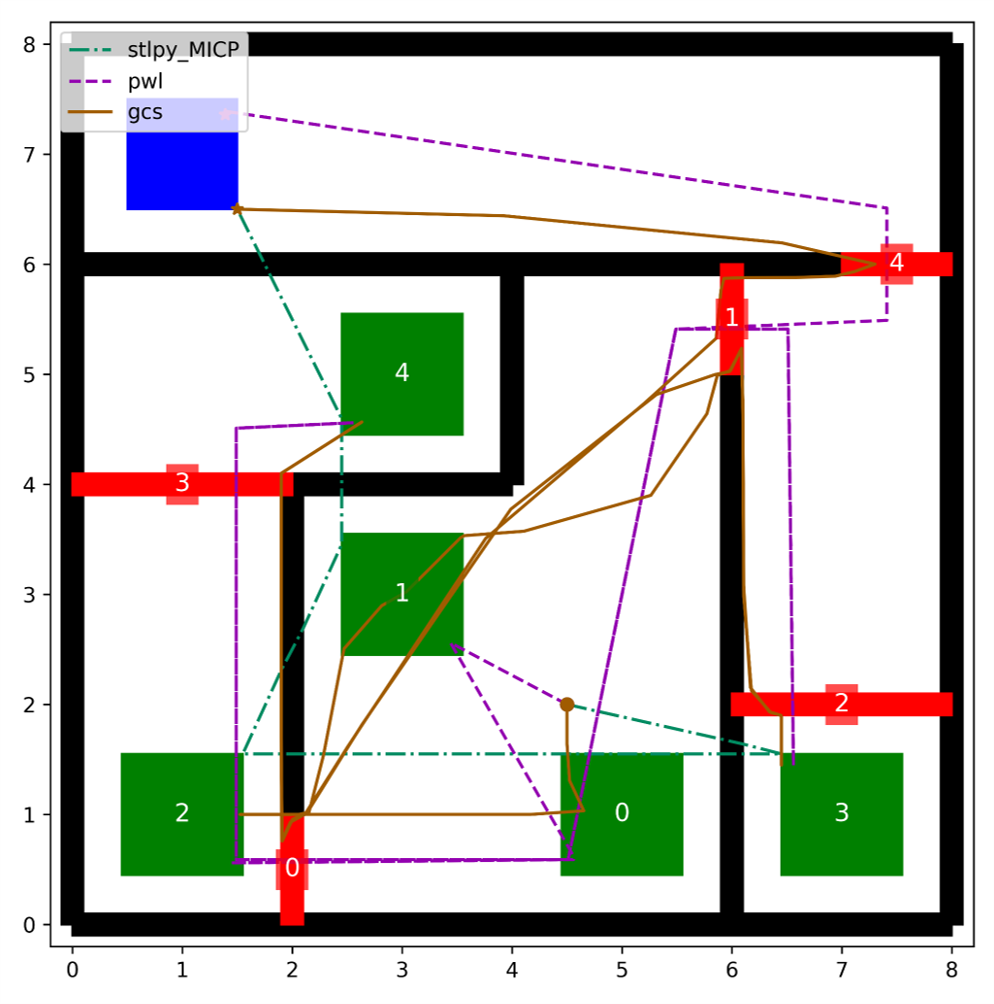}} &
            \subfigure[puzzle-2]{\includegraphics[height=\twoDCaseHeight]{ 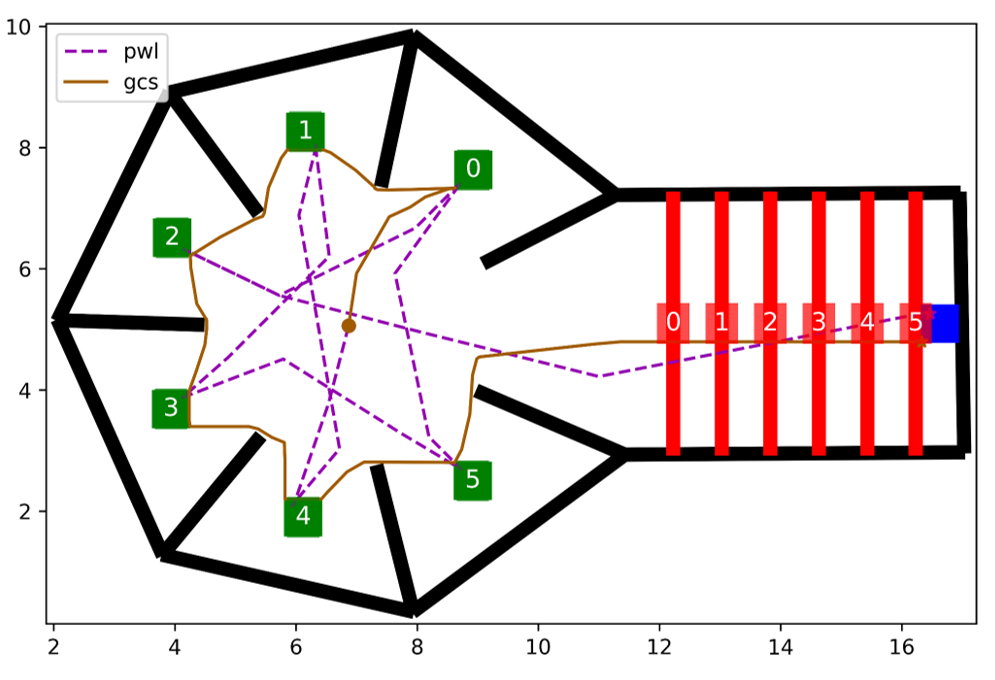}} \\
            \subfigure[rover]{\includegraphics[height=\twoDCaseHeight]{ 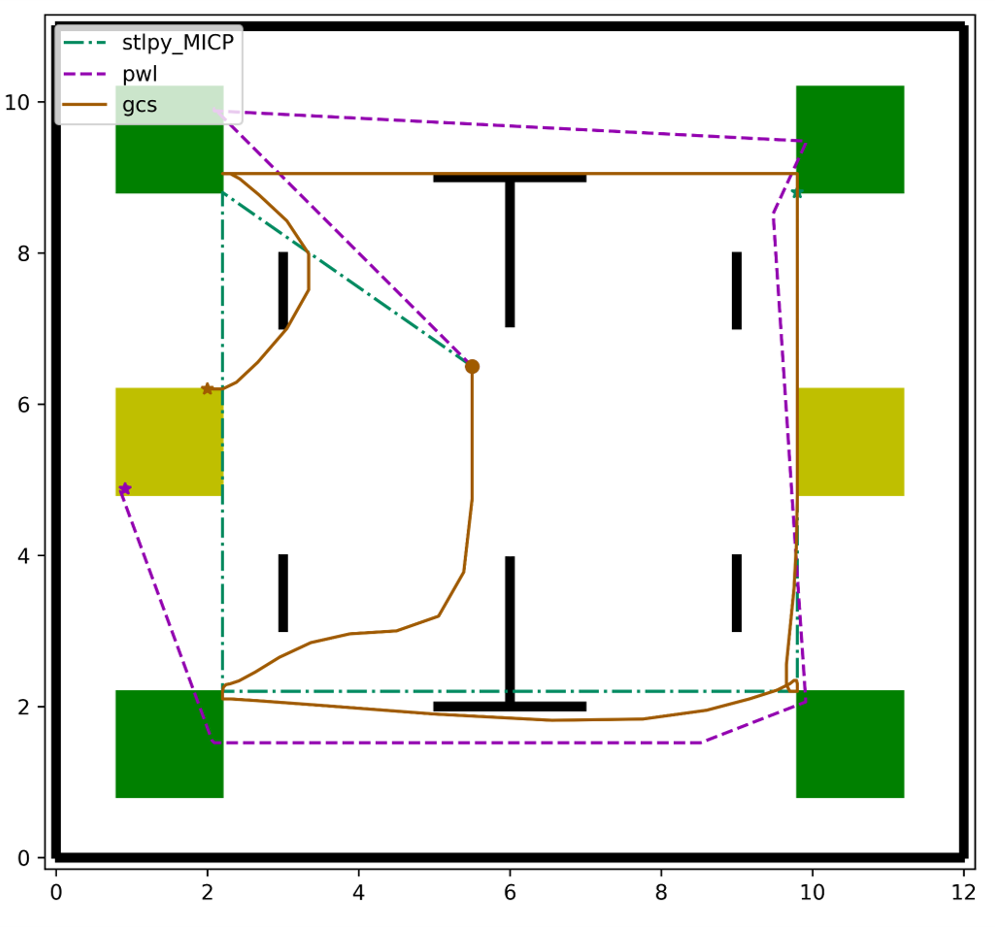}} &
            \subfigure[either-or]{\includegraphics[height=\twoDCaseHeight]{ 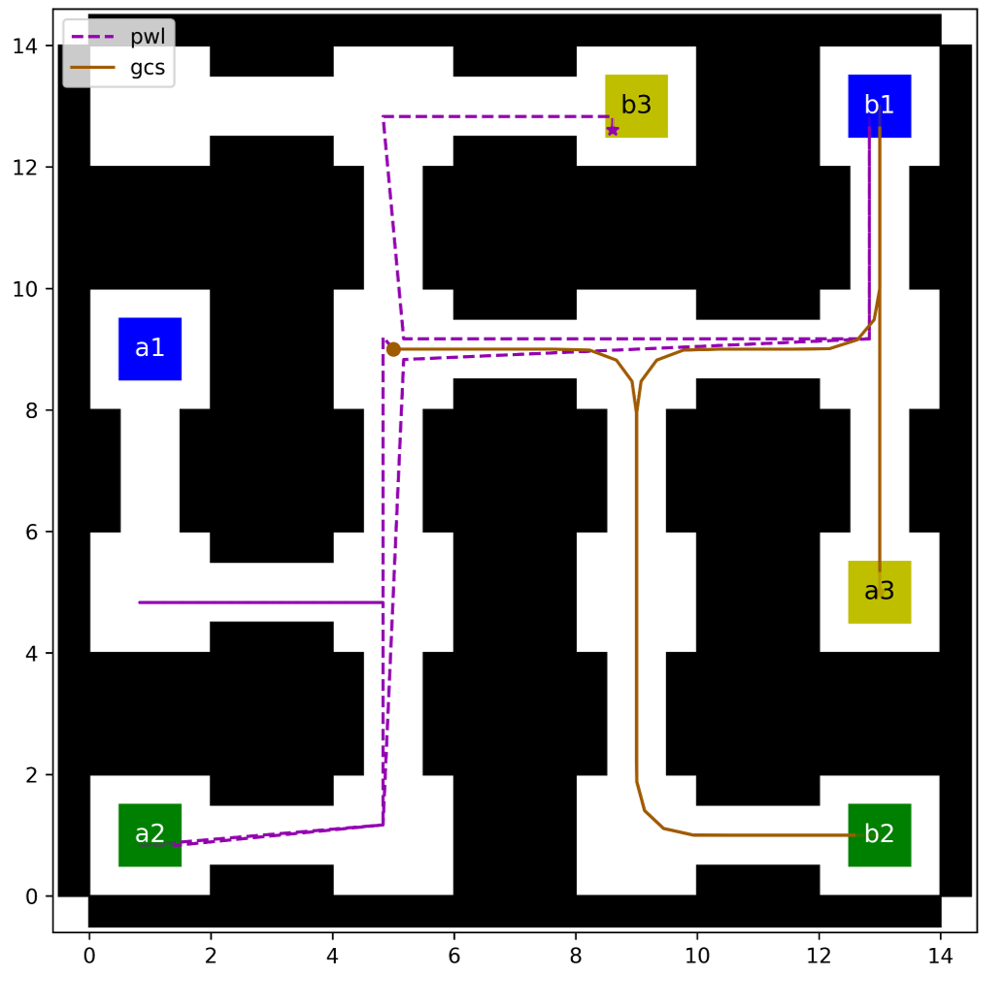}} &
            \subfigure[deliver]{\includegraphics[height=\twoDCaseHeight]{ 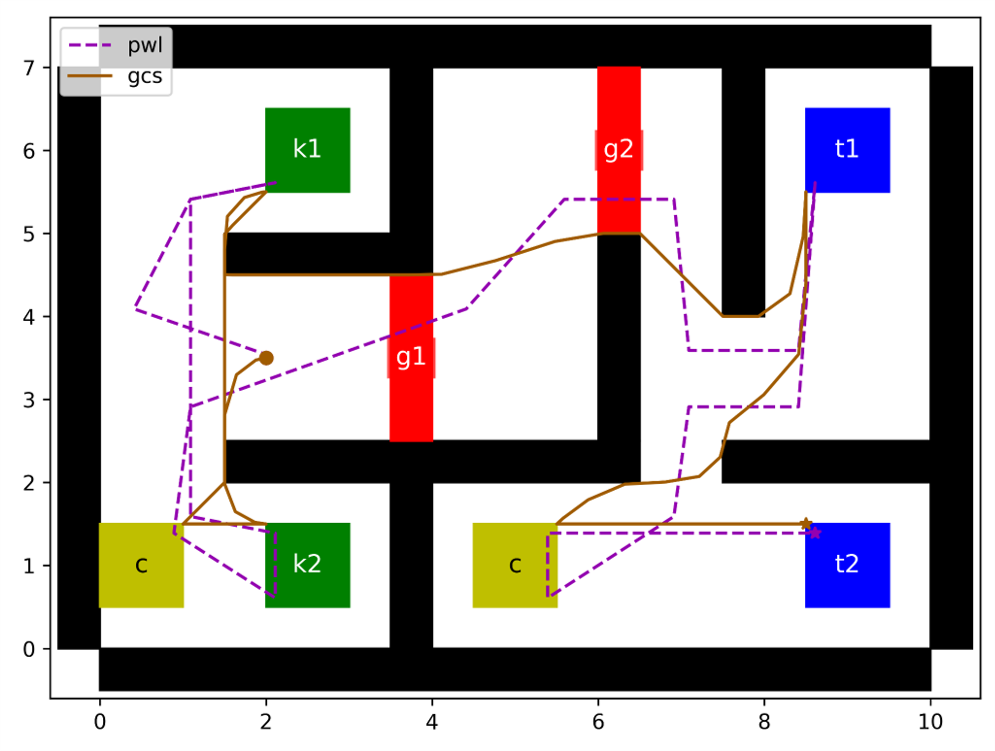}}
            \end{tabular}
        \end{minipage}
    \end{tabular}
    \caption{Benchmarks and computed trajectories of each method for a $2$-D robotic system. GCS denotes the trajectories generated by our method. PWL and stlpy\_MICP denote the trajectories generated by the methods in \cite{sun2022multi} and \cite{kurtz2022mixed}, respectively.}
    \label{fig:2dexample}
\end{figure*}

\subsection{Simulation Experiments on 2-D/3-D STL Planning}
\label{subsec:compareexp}

\textbf{Benchmarks:}
We first evaluate the proposed method on six $2$-D STL motion-planning benchmarks, as shown in Figure~\ref{fig:2dexample}.
The first four benchmarks are adopted from the STL motion-planning literature~\cite{sun2022multi}, while the last two are newly designed to test disjunctive tasks and more complex combinations of temporal requirements.
In all benchmarks, the black region, denoted by $W$, represents a wall that the robot must avoid at all times.
For readability, when presenting the STL specifications below, we omit the common safety constraint $\mathbf{G}_{[0,T]}\neg W$.

In \texttt{stlcg}, shown in Figure~\ref{fig:2dexample}(a), the robot is required to stay in the red region $R$ and the green region $G$ for $5$ seconds each, while always avoiding the blue region $B$.
The STL task can be described as follows
\[
\Phi_1=
\mathbf{F}_{[0,15]}\mathbf{G}_{[0,5]}R
\ \wedge \
\mathbf{F}_{[0,15]}\mathbf{G}_{[0,5]}G
\ \wedge \
\mathbf{G}_{[0,20]}\neg B.
\]

In \texttt{puzzle-1}, shown in Figure~\ref{fig:2dexample}(b), the robot must reach the goal region.
The environment contains several doors, shown in red.
Each door can be crossed only after the robot visits the corresponding numbered green region to collect the key.
Let $G$ denote the goal, $D_1,\dots,D_5$ the doors, and $K_1,\dots,K_5$ the keys.
The STL task is given by
\[
\Phi_2=
\mathbf{F}_{[0,10]}G
\wedge
\left(\bigwedge_{i=1}^5 \neg D_i \mathbf{U}_{[0,10]}K_i  \right).
\]

The benchmark \texttt{puzzle-2}, shown in Figure~\ref{fig:2dexample}(c), is similar to \texttt{puzzle-1} but contains six doors.
The corresponding STL task is given by
\[
\Phi_3=
\mathbf{F}_{[0,30]}G
\wedge
\left(\bigwedge_{i=1}^6 \neg D_i \mathbf{U}_{[0,30]}K_i  \right).
\]

The benchmark \texttt{rover}, shown in Figure~\ref{fig:2dexample}(d), is used to demonstrate that our framework can also handle STL tasks outside the fragment in \eqref{eq:consideredstl} by introducing additional TA templates.
In this scenario, the robot must visit each green goal region $G_i$.
Moreover, after arriving at any goal region, it must reach the yellow region $S$ within the following time interval.
The task is
\[
\Phi_4 =
\mathbf{G}_{[0,30]}
\left(
\bigvee_{i=1}^4G_i
\implies
\mathbf{F}_{[0,10]} S
\right)
\wedge
\left(
\bigwedge_{i=1}^4 \mathbf{F}_{[0,30]}G_i
\right).
\]
Since $\Phi_4$ is not directly contained in the STL fragment \eqref{eq:consideredstl}, we construct a new TA template for this task, as discussed in the appendix.

\begin{figure*}[t]
	\subfigure[3-D Scenario.] 
	{\label{fig:3d-scenario}
 \begin{minipage}[b]{0.22\linewidth}
	\centering
 \raisebox{0.51cm}{\includegraphics[height=4.16cm]{ 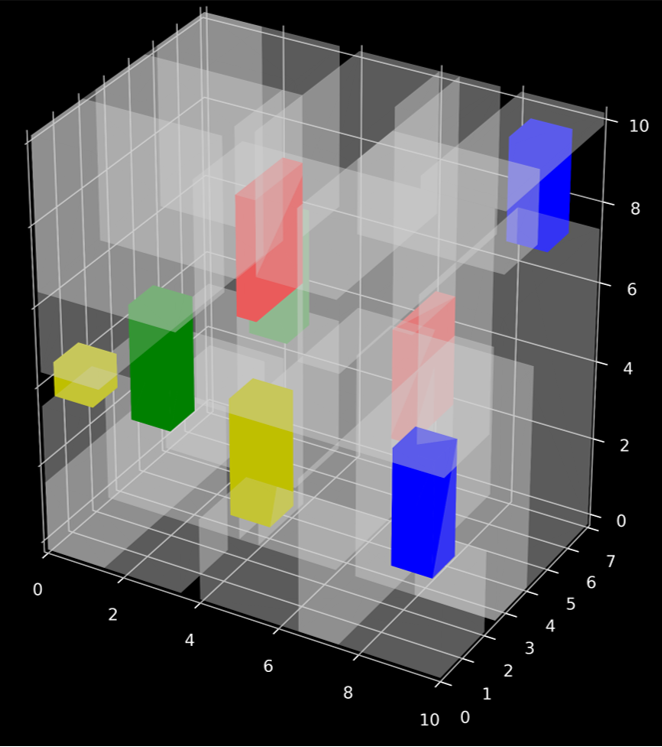}}
    \end{minipage}
		}\subfigure[Projection of the trajectories computed by our approach and the PWL method in~\cite{sun2022multi}] 
		{\label{fig:3dtrajectory}
 \begin{minipage}[b]{0.75\linewidth}
		\centering
 \includegraphics[height=5.3cm]{ 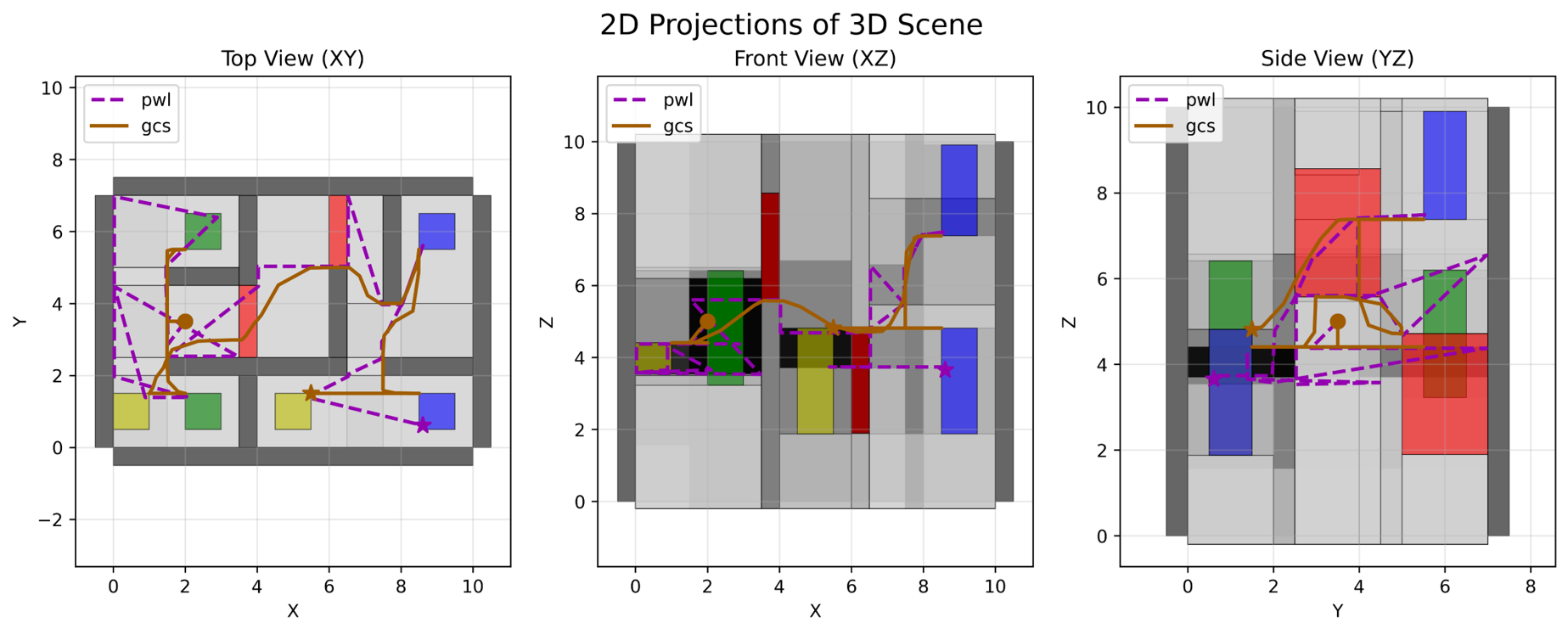}
    \end{minipage}
		}
    \caption{Results for the $3$-D quadrotor experiment.}
   \label{fig:3Dresult}
\end{figure*}

The benchmark \texttt{either-or}, shown in Figure~\ref{fig:2dexample}(e), considers an STL task with disjunctions.
There are three target pairs $(a_i,b_i)$, $i=1,2,3$.
For each pair, the robot must reach at least one target within the specified time interval:
\[
\Phi_5=
\bigwedge_{i=1}^2
\left(
\mathbf{F}_{[0,20]}a_i
\vee
\mathbf{F}_{[0,20]}b_i
\right)
\wedge
\left(
\mathbf{F}_{[20,30]}a_3
\vee
\mathbf{F}_{[20,30]}b_3
\right).
\]

Finally, the benchmark \texttt{deliver}, shown in Figure~\ref{fig:2dexample}(f), requires the robot to recharge in the yellow region $c$ and reach targets $t_1$ and $t_2$ within their corresponding time intervals.
In addition, the robot must collect keys $k_1$ and $k_2$ to open doors $g_1$ and $g_2$, respectively.
The STL task is
\begin{equation}
\label{eq:stlfordeliver}
\begin{aligned}
\Phi_6
=
&\ \mathbf{G}_{[0,20]}\mathbf{F}_{[0,10]}c
\wedge
\left(
\bigwedge_{i=1}^2
\neg g_i\mathbf{U}_{[2,8]}k_i
\right)  \\
&\wedge
\mathbf{F}_{[10,20]} t_1
\wedge
\mathbf{F}_{[20,30]} t_2.
\end{aligned}
\end{equation}

\textbf{Comparison methods:}
We compare our approach with two representative trajectory-optimization-based methods for STL motion planning.
The first baseline is the piecewise-linear (PWL) method in~\cite{sun2022multi}, which computes a piecewise-linear path satisfying the STL task and assumes the availability of a tracking controller for execution.
The second baseline is the MPC-based method in~\cite{belta2019formal,kurtz2022mixed}, which considers STL specifications under discrete-time semantics and solves an MPC problem with STL constraints to compute control inputs.
Among these two baselines, only the MPC method explicitly incorporates dynamic-feasibility constraints.

In contrast, for differentially flat systems, our method can enforce dynamic feasibility by imposing smoothness constraints on the generated B\'ezier trajectory.
Although both comparison methods also lead to MICPs, their convex relaxations are generally weak because they rely on big-$M$ encodings of logical constraints.
Moreover, these MICPs do not admit an inexpensive rounding procedure comparable to the GCS relaxation-and-rounding strategy.
In the experiments, 
the MPC method is implemented based on \textsf{stlpy}~\cite{kurtz2022mixed}.

\textbf{Results on 2-D benchmarks:}
The computed trajectories are shown in Figure~\ref{fig:2dexample}, and the runtime results are reported in Table~\ref{tab:runningtime}.
For the MPC baseline, we use a single-integrator model.
Under the same setting, our method enforces dynamic feasibility by requiring the generated trajectory to be once differentiable.

The MPC method is the fastest on \texttt{stlcg}, but it performs poorly on the more complex benchmarks.
It fails to find a solution within $7200$ seconds for \texttt{puzzle-2}, \texttt{either-or}, and \texttt{deliver}.
On \texttt{puzzle-1} and \texttt{rover}, it returns a trajectory only after substantially longer computation times than both our method and the PWL method.
Moreover, because the MPC baseline uses discrete-time STL semantics, the resulting trajectory may violate the continuous-time specification between sampling instants, for example by passing through wall regions.

Compared with the PWL method, our approach successfully solves all benchmarks with comparable or better runtime in most cases.
In particular, on \texttt{puzzle-1}, \texttt{puzzle-2}, \texttt{rover}, and \texttt{deliver}, our method is substantially faster than the PWL method.
The PWL method is faster on \texttt{stlcg} and \texttt{either-or}.
In terms of trajectory quality, our method produces smoother continuous-time trajectories and performs better on \texttt{puzzle-2}, \texttt{either-or}, and \texttt{deliver}, where the PWL method produces noticeably longer paths.
Overall, these results show that the proposed GCS formulation is competitive on low-dimensional STL planning problems while providing continuous-time smooth trajectories.

\textbf{Results on a 3-D quadrotor:}
We next extend the \texttt{deliver} benchmark to a $3$-D setting.
For each feasible region in \texttt{deliver}, we assign a valid interval along the $z$-axis.
The wall regions remain obstacles regardless of the $z$ coordinate.
The resulting scenario is shown in Figure~\ref{fig:3Dresult}.
For clarity, the wall regions are omitted from the $3$-D visualization.
The STL specification is the same as $\Phi_6$ in \eqref{eq:stlfordeliver}.

We solve this problem for a quadrotor with fixed yaw using our method and the PWL method.
When using our approach, we require the trajectory to be four times differentiable.
This requirement ensures dynamic feasibility for the quadrotor because the quadrotor dynamics are differentially flat~\cite{mellinger2011minimum}.
As shown in Table~\ref{tab:runningtime}, our method solves the problem in $15.68$ seconds, whereas the PWL method requires $489$ seconds.
Moreover, the PWL method only returns a piecewise-linear path, so an additional tracking controller is required for execution.

To isolate the effect of the smoothness constraint, we also tested our method without the smoothness requirement.
In this case, the solve times of our method on the $2$-D \texttt{deliver} benchmark and the $3$-D quadrotor benchmark are $0.86$ seconds and $4.8$ seconds, respectively.
By comparison, the PWL method requires $2.76$ seconds and $489$ seconds on the same two problems.
These results indicate that the proposed GCS formulation scales favorably with the configuration-space dimension and is well suited for continuous-time STL motion planning of robotic systems.

\subsection{Simulation Experiments on High-Dimensional Humanoid}
\label{subsection:highdimexp}

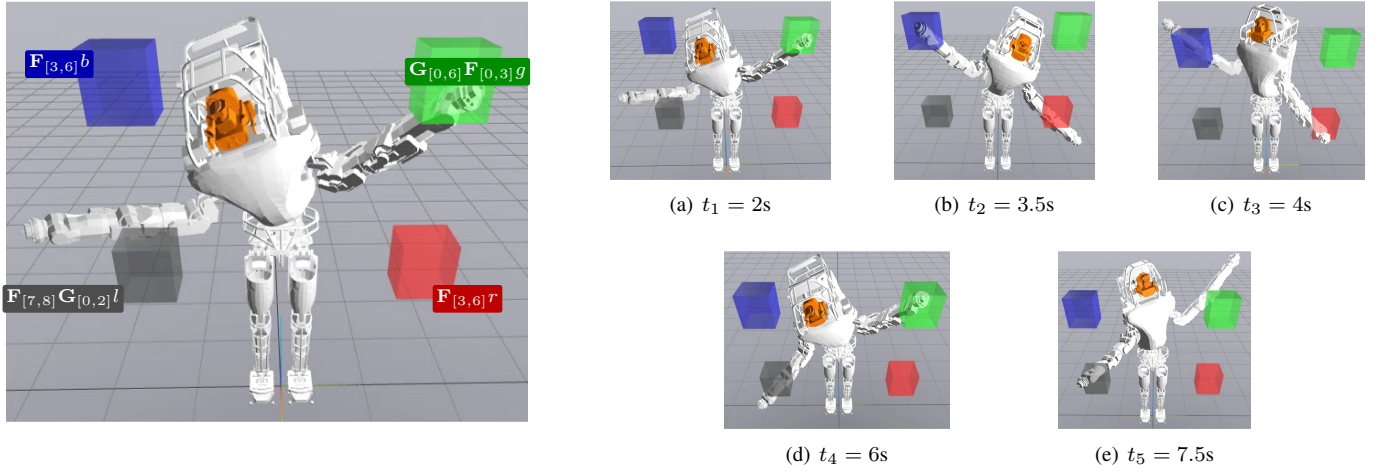
\begin{figure*}[t]
    \centering
    \begin{minipage}[t]{0.38\linewidth}
        \vspace{0pt}
        \centering
        \begin{tikzpicture}
            \node[inner sep=0pt] (scene) {\includegraphics[width=\linewidth]{ 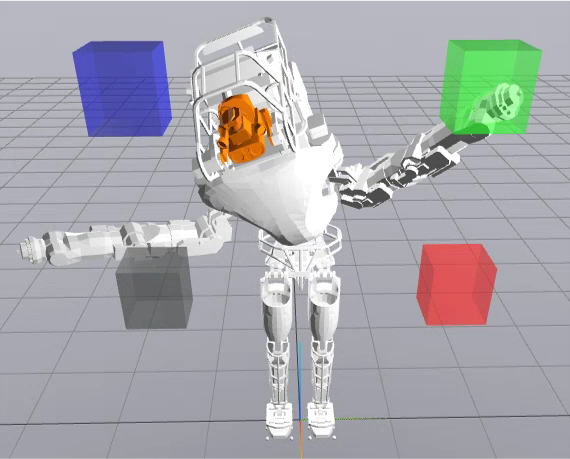}};
            \node[fill=green!55!black,text=white,rounded corners=1pt,inner sep=1.5pt,font=\scriptsize,align=center] at (2.65,1.85)
            {$\mathbf{G}_{[0,6]}\mathbf{F}_{[0,3]}g$};
            \node[fill=blue!70!black,text=white,rounded corners=1pt,inner sep=1.5pt,font=\scriptsize,align=center] at (-2.75,1.95)
            {$\mathbf{F}_{[3,6]}b$};
            \node[fill=red!75!black,text=white,rounded corners=1pt,inner sep=1.5pt,font=\scriptsize,align=center] at (2.65,-1.15)
            {$\mathbf{F}_{[3,6]}r$};
            \node[fill=black!70,text=white,rounded corners=1pt,inner sep=1.5pt,font=\scriptsize,align=center] at (-2.70,-1.15)
            {$\mathbf{F}_{[7,8]}\mathbf{G}_{[0,2]}l$};
        \end{tikzpicture}
    \end{minipage}\hfill
    \begin{minipage}[t]{0.55\linewidth}
        \vspace{-0.50em}
        \centering
        \subfigure[$t_1=2$s]{\includegraphics[height=0.096\textheight]{ 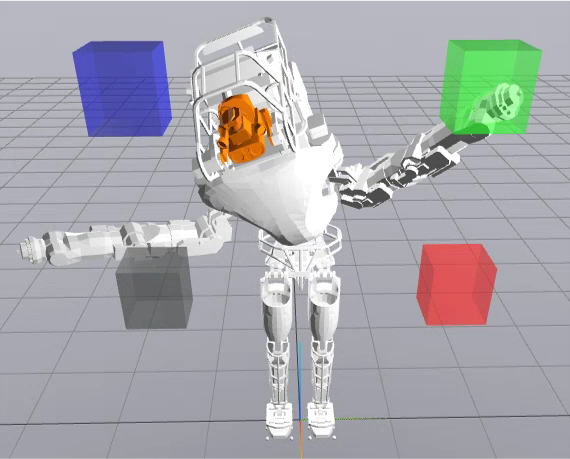}}
        \hfill
        \subfigure[$t_2=3.5$s]{\includegraphics[height=0.096\textheight]{ 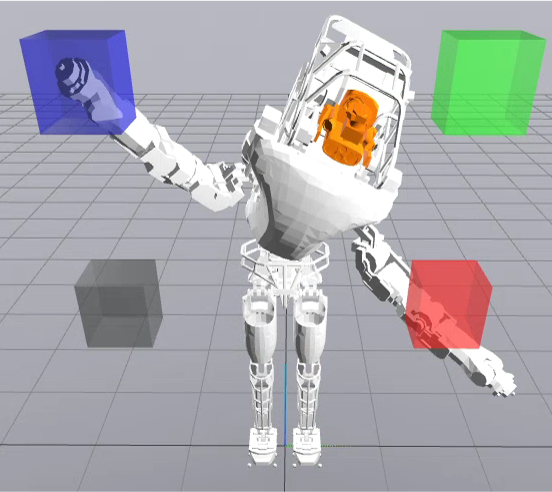}}
        \hfill
        \subfigure[$t_3=4$s]{\includegraphics[height=0.096\textheight]{ 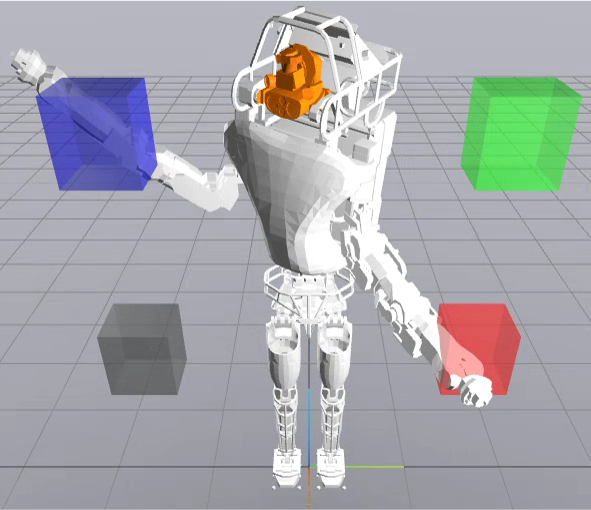}}

        \vspace{0.35em}
        \subfigure[$t_4=6$s]{\includegraphics[height=0.096\textheight]{ 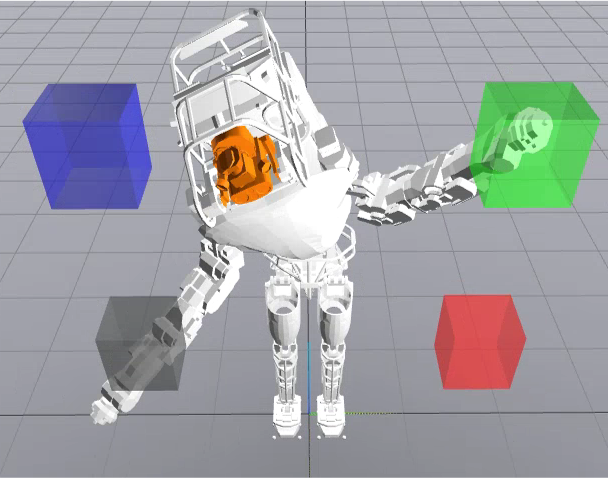}}
        \hspace{0.12\linewidth}
        \subfigure[$t_5=7.5$s]{\includegraphics[height=0.096\textheight]{ 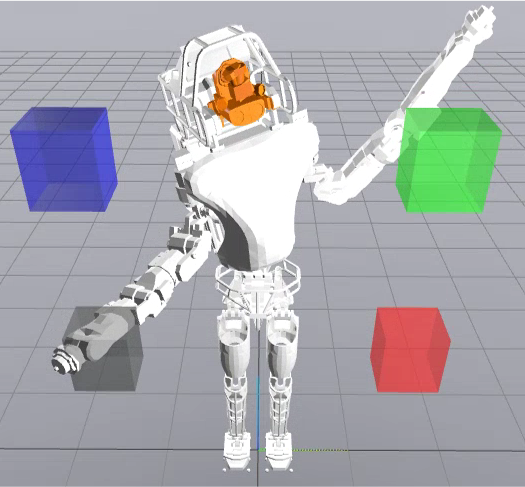}}
    \end{minipage}
    \vspace{0.35em}
    \caption{Visualization of the humanoid STL task. The left panel shows the annotated target regions, and the right panel shows enlarged key frames along the synthesized trajectory. The key frames verify the specification: $t_1$ and $t_4$ witness the repeated visits to $g$ in the two $3$-second subwindows of $[0,6]$; $t_2$ reaches $b$ and $t_3$ reaches $r$ inside $[3,6]$; and $t_5$ starts the $2$-second dwell in $l$.}
    \label{fig:30dofexample}
\end{figure*}

We next consider STL motion planning for the $30$-DoF Atlas humanoid model from Drake~\cite{drake}.
This experiment is designed to demonstrate the scalability of the proposed method in a high-dimensional configuration space.
As shown in Figure~\ref{fig:30dofexample}, the environment contains four target boxes colored green, blue, red, and black.
The humanoid must reach these boxes under different temporal constraints.
Specifically, over the time interval $[0,6]$, the robot must reach the upper-right green box at least once every $3$ seconds.
In addition, the blue and red boxes must be touched during the time interval $[3,6]$.
Finally, the robot must remain in the black box for $2$ seconds, with the stay starting at some time in $[7,8]$.
The STL task is
\[
\Phi=
\mathbf{G}_{[0,6]}\mathbf{F}_{[0,3]}g
\wedge
\mathbf{F}_{[3,6]}b
\wedge
\mathbf{F}_{[3,6]}r
\wedge
\mathbf{F}_{[7,8]}\mathbf{G}_{[0,2]}l,
\]
where $g$, $b$, $r$, and $l$ denote the regions in which the robot reaches the green, blue, red, and black targets, respectively.

We use inverse kinematics and C-IRIS~\cite{dai2024certified} to compute convex regions of the configuration space.
The green and red targets correspond to left-hand reaching constraints, while the blue and black targets correspond to right-hand reaching constraints.
We also construct an additional convex region without kinematic target constraints.
As in~\cite{kurtz2023temporal}, we do not model contact interactions with the ground or the contact wrench cone.
The pelvis pose is fixed in the world frame, so the robot is treated as a $30$-DoF system.

We use B\'ezier splines of degree $K=3$ and impose a $\mathcal{C}_2$ continuity constraint.
By solving the convex relaxation, the proposed method obtains a feasible trajectory in $10.2$ seconds, as reported in Table~\ref{tab:runningtime}.
The resulting trajectory, shown in Figure~\ref{fig:30dofexample}, satisfies the STL task.
This experiment demonstrates that the proposed framework can handle STL motion-planning problems for high-dimensional robotic systems with complex temporal requirements.

\subsection{Hardware Experiment on Robot Arm}
\label{subsection:hardwardexp}
\begin{figure*}[t]
    \centering
    \newsavebox{\hardwareEnvBox}
    \newlength{\hardwareLeftWidth}
    \newlength{\hardwareRightWidth}
    \newlength{\hardwarePanelHeight}
    \newlength{\hardwareLeftDrop}
    \setlength{\hardwareLeftWidth}{0.526\linewidth}
    \setlength{\hardwareRightWidth}{0.445\linewidth}
    \setlength{\hardwarePanelHeight}{0.2422\linewidth}
    \setlength{\hardwareLeftDrop}{0.010\linewidth}
    \sbox{\hardwareEnvBox}{%
        \begin{tikzpicture}
            \node[inner sep=0pt] (env) {\includegraphics[width=\hardwareLeftWidth]{ 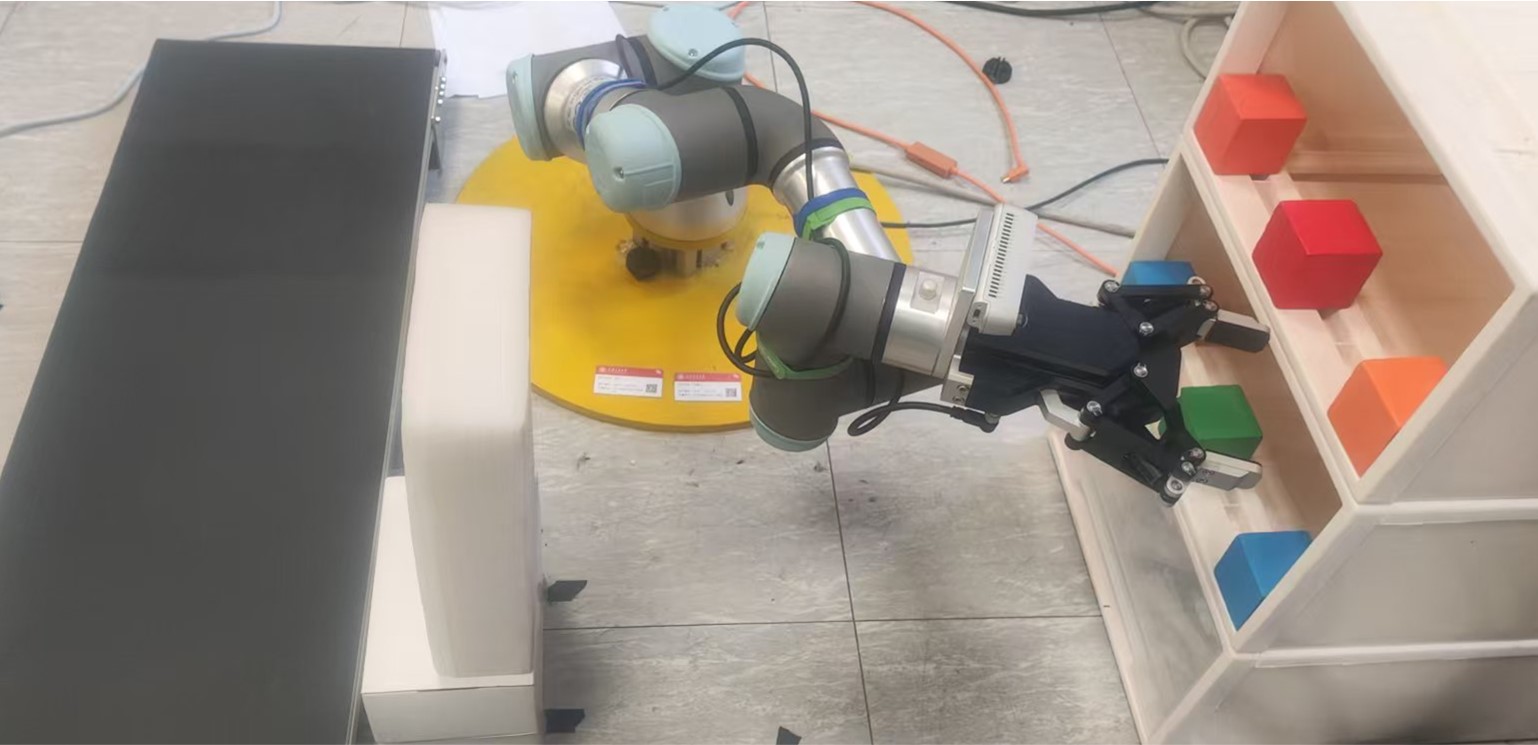}};
            \node[fill=white,fill opacity=0.86,text opacity=1,rounded corners=1pt,inner sep=1.5pt,font=\scriptsize,align=center] at (-3.05,-0.72) {conveyor\\place $(b)$};
            \node[fill=white,fill opacity=0.86,text opacity=1,rounded corners=1pt,inner sep=1.5pt,font=\scriptsize,align=center] at (-1.20,-1.25) {obstacle\\avoid};
            \node[fill=white,fill opacity=0.86,text opacity=1,rounded corners=1pt,inner sep=1.5pt,font=\scriptsize,align=center] at (3.45,1.35) {shelf\\pick items};
            \node[fill=green!60!black,text=white,rounded corners=1pt,inner sep=1.5pt,font=\scriptsize] at (3.10,-0.68) {green item ($a$)};
            \node[fill=red!75!black,text=white,rounded corners=1pt,inner sep=1.5pt,font=\scriptsize] at (3.30,0.56) {red item ($c$)};
        \end{tikzpicture}%
    }
    \begin{minipage}[t]{\hardwareLeftWidth}
        \vspace{\hardwareLeftDrop}
        \centering
        \usebox{\hardwareEnvBox}
    \end{minipage}\hfill
    \begin{minipage}[t][\hardwarePanelHeight][s]{\hardwareRightWidth}
        \vspace{0pt}
        \centering
        \makebox[\linewidth][c]{%
        \subfigure[$t_1=10$s]{\includegraphics[height=0.069\textheight]{ 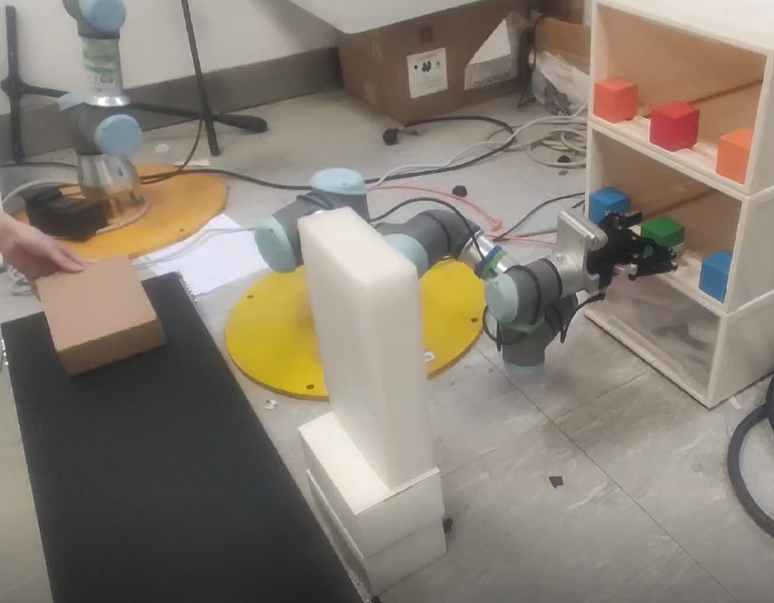}}
        \hspace{0.025\linewidth}
        \subfigure[$t_2=23$s]{\includegraphics[height=0.069\textheight]{ 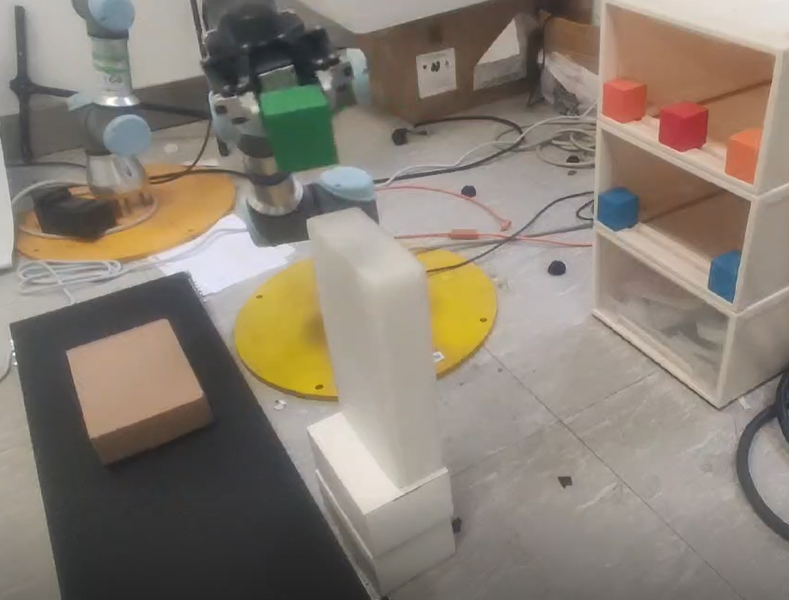}}
        \hspace{0.025\linewidth}
        \subfigure[$t_3=27$s]{\includegraphics[height=0.069\textheight]{ 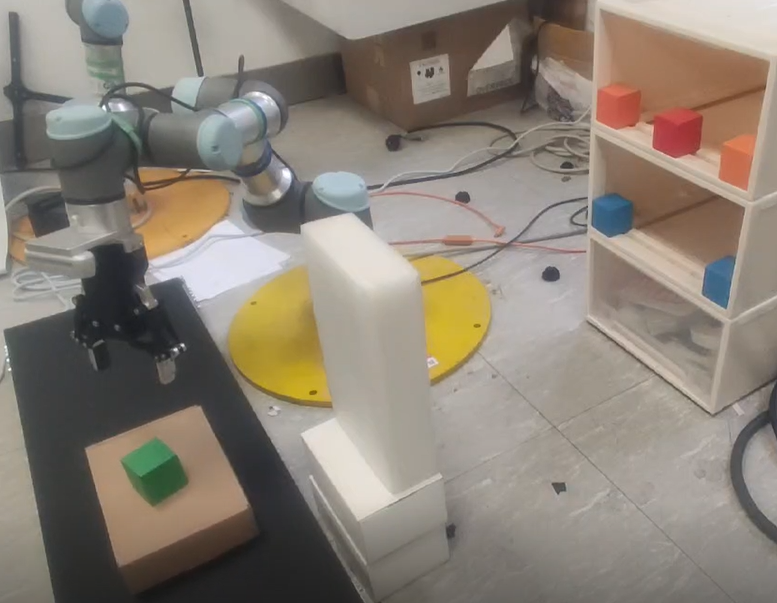}}}
        \par\vfill
        \makebox[\linewidth][c]{%
        \subfigure[$t_4=42$s]{\includegraphics[height=0.069\textheight]{ 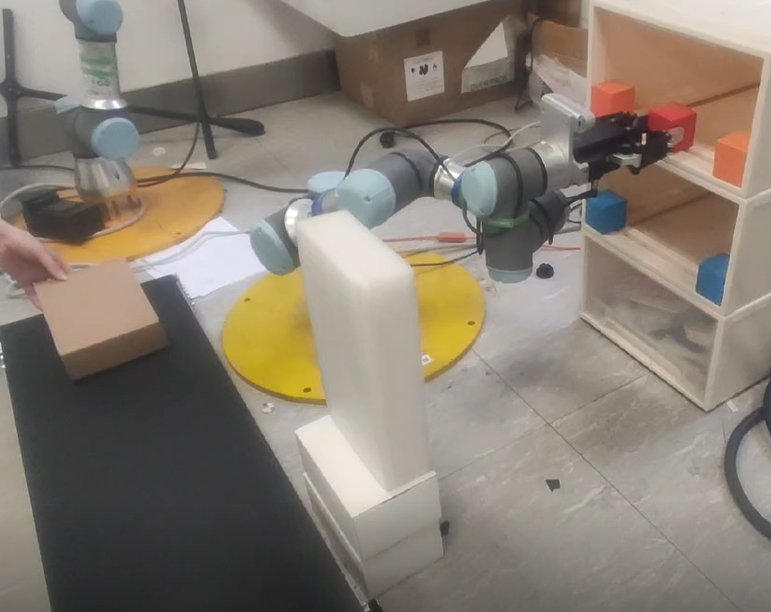}}
        \hspace{0.025\linewidth}
        \subfigure[$t_5=53$s]{\includegraphics[height=0.069\textheight]{ 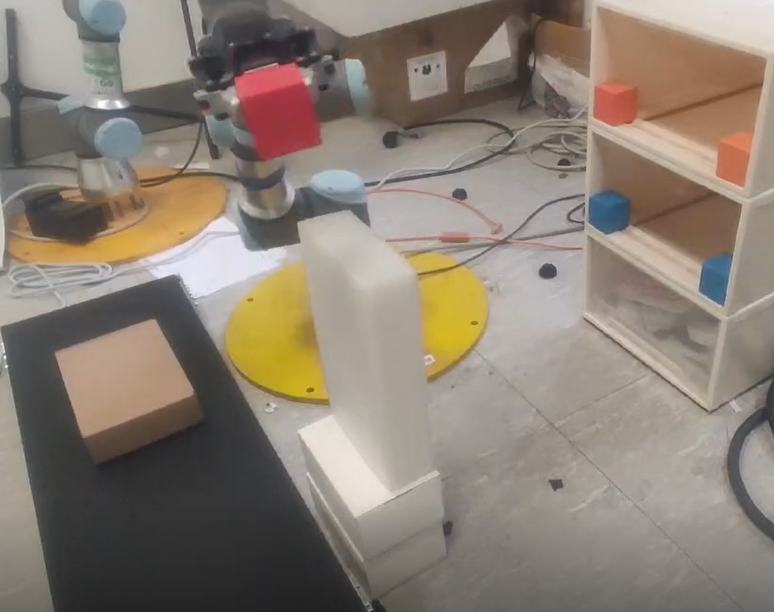}}
        \hspace{0.025\linewidth}
        \subfigure[$t_6=58$s]{\includegraphics[height=0.069\textheight]{ 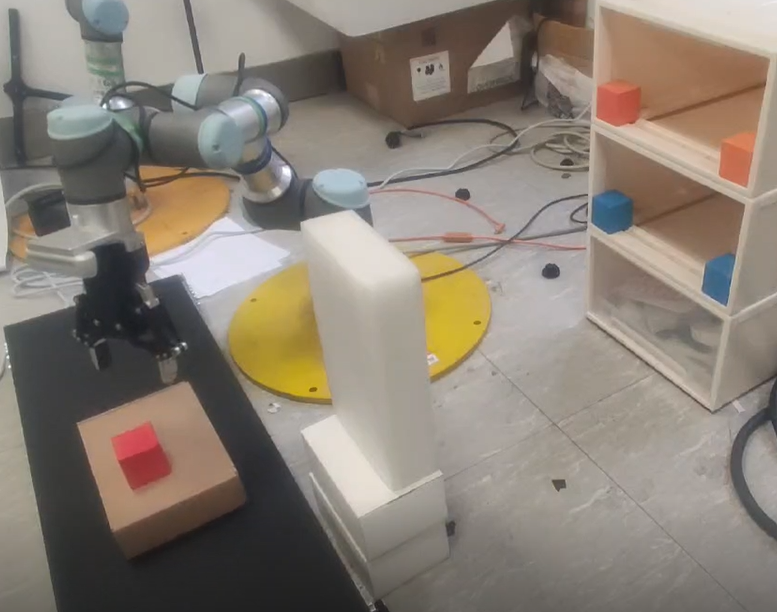}}}
    \end{minipage}
    \caption{Visualization of the hardware experiment under the STL task. The left panel shows the hardware setup, where the green item corresponds to region $a$, the red item corresponds to region $c$, and both items should be placed on the conveyor at region $b$ while the middle obstacle is avoided. The right panel shows six representative key frames: $t_1$ shows grasping the green item at $a$, $t_2$ shows carrying it across the obstacle, and $t_3$ shows successfully placing it at $b$; $t_4$ shows grasping the red item at $c$, $t_5$ shows carrying it across the obstacle, and $t_6$ shows successfully placing it at $b$.}
    \label{fig:hardware}
    \label{fig:hardware-scenario}
\end{figure*}
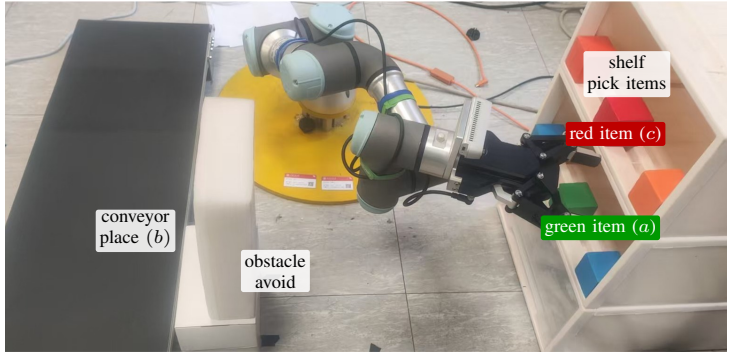
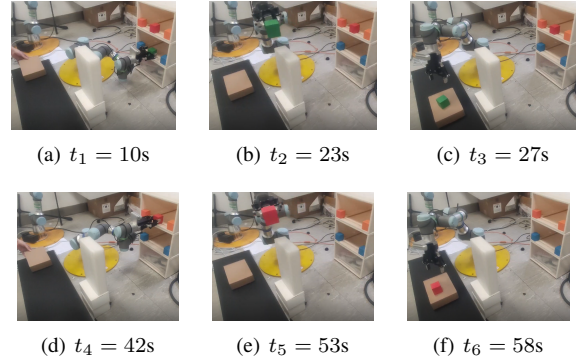

Finally, we validate the proposed approach on a hardware platform using a $6$-DoF model of a UR-3 robot arm.
The experimental setup is shown in Figure~\ref{fig:hardware-scenario}.
The right side contains a shelf from which the manipulator must pick up two items in sequence, namely the green and red items.
The middle part contains an obstacle that must be avoided throughout the task.
The left side contains a conveyor belt on which the items must be placed within prescribed time intervals.

The STL task is
\[
\begin{aligned}
\Phi
=&\
\mathbf{F}_{[0,5.5]}e
\wedge
\mathbf{F}_{[6,10]}\mathbf{G}_{[0,2]}a
\wedge
\mathbf{F}_{[13.5,15]}e \\
&\wedge
\mathbf{F}_{[21,23]}\mathbf{G}_{[0,4]}b
\wedge
\mathbf{F}_{[33,35]}f
\wedge
\mathbf{F}_{[36,40]}\mathbf{G}_{[0,2]}c \\
&\wedge
\mathbf{F}_{[43.5,45.5]}f
\wedge
\mathbf{F}_{[51,56]}\mathbf{G}_{[0,4]}b .
\end{aligned}
\]
Here, regions $a$ and $c$ are the grasping locations for the green and red items, respectively.
Region $b$ is the placement location for both items.
Regions $e$ and $f$ are auxiliary positions that allow the manipulator to approach and grasp the green and red items more reliably.
The STL formula requires the robot arm to visit auxiliary regions before and after grasping, and to remain in the grasping and placement regions for prescribed durations so that the end effector can reliably grasp and release the items.

We use inverse kinematics and C-IRIS~\cite{dai2024certified} to compute convex regions of the configuration space corresponding to each task region.
We then search for a trajectory satisfying the STL task using B\'ezier splines of degree $K=8$ under a $\mathcal{C}_4$ continuity constraint.
By solving the convex relaxation, the proposed method obtains a feasible trajectory in $9.65$ seconds.
As shown in Figure~\ref{fig:hardware}, the computed trajectory satisfies the STL specification and can be executed on the hardware platform.

For comparison, the method in~\cite{kurtz2020trajectory} considers the STL task $\mathbf{G}_{[40,50]}(a\vee b)$ for a $7$-DoF robot arm and requires $160$ seconds to find a solution.
In contrast, our approach handles a more complex STL task on a real robot platform while finding a solution in substantially less time.
This hardware experiment further demonstrates the practical efficiency of the proposed GCS-based STL motion-planning framework.

\section{Conclusion}\label{sec:con}
This paper presented a continuous-time STL motion-planning framework based on timed automata and graphs of convex sets.
The proposed approach converts logical task progress, real-time constraints, convex-region occupancy, smoothness requirements, and velocity bounds into a unified GCS shortest-path formulation.
The resulting solution directly reconstructs a B\'ezier-spline trajectory, thereby avoiding dense time discretization while retaining continuous-time satisfaction guarantees.
We proved the soundness of the construction. We also analyzed the computational efficiency of the framework.
To improve the scalability of the automaton layer, we further introduced a compact timed-automaton construction for an expressive STL fragment based on constant-size temporal templates and Boolean composition.
The effectiveness of the proposed framework was validated through extensive experiments.
On low-dimensional STL planning benchmarks, the method achieved competitive or better runtime compared with existing optimization-based approaches while generating smooth continuous-time trajectories.
The $3$-D quadrotor, $30$-DoF humanoid, and UR-3 robot arm experiments demonstrated favorable scalability to higher-dimensional systems.  
Future work will investigate richer STL fragments, robustness under uncertainty, and extensions to multi-agent robotic systems.

\bibliographystyle{ieeetr}
\bibliography{main}
\appendix
\subsection{Proof of Proposition~\ref{prop:productandunion}}
\begin{proof}[\bf Proof of Proposition~\ref{prop:productandunion}]
We first prove 1).
For the ``$\Longrightarrow$'' direction, suppose that $\xi \models A_1\bigsqcup A_2$.
Then there exists an accepting run $\rho$ of $A_1\bigsqcup A_2$ that induces $\xi$.
Since $S_1\cap S_2=\emptyset$ and the union construction introduces no transition between $A_1$ and $A_2$, the run $\rho$ is entirely contained in either $A_1$ or $A_2$.
Thus, $\xi\models A_1$ or $\xi\models A_2$, which implies $\xi\models \Phi_1\vee\Phi_2$.
Conversely, if $\xi\models \Phi_1\vee\Phi_2$, then $\xi\models\Phi_1$ or $\xi\models\Phi_2$.
By the assumed converse direction for $A_1$ and $A_2$, there is an accepting run of either $A_1$ or $A_2$ that induces $\xi$.
The same run is also an accepting run of $A_1\bigsqcup A_2$, and hence $\xi\models A_1\bigsqcup A_2$.

We next prove the ``$\Longrightarrow$'' direction in 2).
Suppose that $\xi\models A_1\bigotimes A_2$, and let $\rho$ be an accepting run of $A_1\bigotimes A_2$ inducing $\xi$.
Write each product state as $s_j=(s_j^1,s_j^2)$.
By projecting $\rho$ onto its first component, deleting all right transitions, and merging the corresponding time intervals, we obtain a run $\rho_1$ of $A_1$.
Indeed, right transitions do not change the first component and, because $C_1\cap C_2=\emptyset$, they do not reset any clock in $C_1$.
Therefore, the clock valuations and guards along the remaining left transitions are exactly those of a valid run of $A_1$.
Since $\rho$ is accepting in the product, its final state belongs to $S_{1,F}\times S_{2,F}$; hence the projected run $\rho_1$ is accepting in $A_1$.
Moreover, $\rho_1$ still induces $\xi$: whenever $\xi(\tau)\in M((s_j^1,s_j^2))$, we have
\[
M((s_j^1,s_j^2))=M_1(s_j^1)\cap M_2(s_j^2)\subseteq M_1(s_j^1).
\]
Thus, $\xi\models A_1$.
The same projection argument applied to the second component yields $\xi\models A_2$.
Using the assumed implications $\xi\models A_i\Rightarrow \xi\models\Phi_i$, we obtain $\xi\models\Phi_1\wedge\Phi_2$.

Finally, we prove the ``$\Longleftarrow$'' direction in 2).
Assume that $\xi\models\Phi_1\wedge\Phi_2$.
By the assumed converse direction, there exist accepting runs $\rho_1$ and $\rho_2$ of $A_1$ and $A_2$, respectively, both inducing $\xi$.
We construct a run $\rho$ of $A_1\bigotimes A_2$ by ordering all transition times of $\rho_1$ and $\rho_2$.
At each such time, the product run takes a left transition if the next transition comes from $\rho_1$, and a right transition if it comes from $\rho_2$; if two transition times coincide, either order gives a valid run because $C_1\cap C_2=\emptyset$.
The state of $\rho$ is always the pair of the current states of $\rho_1$ and $\rho_2$, and the clock valuations are obtained by combining the valuations of the two component runs.
By Definition~\ref{def:productofTA}, every transition of $\rho$ is therefore a valid product transition.
The final state of $\rho$ belongs to $S_{1,F}\times S_{2,F}$, so $\rho$ is accepting.
Furthermore, since both $\rho_1$ and $\rho_2$ induce $\xi$, for every time $\tau$ the trajectory lies in both component regions, and hence in their intersection $M_1(s^1)\cap M_2(s^2)=M((s^1,s^2))$.
Thus, $\rho$ induces $\xi$, which proves $\xi\models A_1\bigotimes A_2$.
This completes the proof.
\end{proof}
\subsection{TA Construction for Formula $\Phi_4$}
\begin{figure}[t]
    \centering
    \includegraphics[width=\linewidth]{ 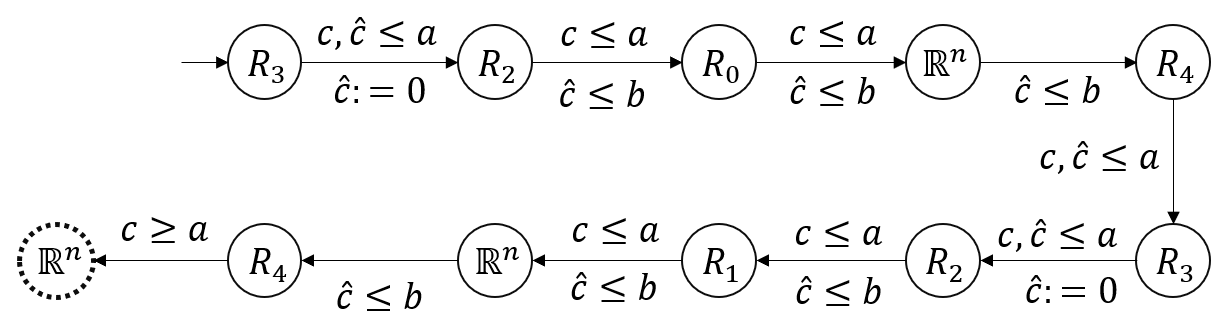}
    \caption{The TA for $\hat{\Phi}$}
    \label{fig:newTA}
\end{figure}

We next illustrate how to construct a TA for the STL task
\begin{equation}\label{eq:newTAdef}
    \hat{\Phi}=\mathbf{G}_{[0,a]}(\bigvee_{i=1}^2G_i \implies \mathbf{F}_{[0,b]} S) \wedge (\bigwedge_{i=1}^2 \mathbf{F}_{[0,a]}G_i).
\end{equation}
The TA for $\hat{\Phi}$ in \eqref{eq:newTAdef} is $A_{\hat{\Phi}}=(S,S_0,R,C,\delta,M,S_F)$ is  shown in Figure~\ref{fig:newTA}, where
\begin{itemize}[leftmargin=*,topsep=0pt,itemsep=3pt]
    \item $S=\{ s_0,s_1,\dots,s_{10}\}$, $S_0=\{s_0 \}$, and $S_F=\{s_{10}\}$;
    \item $R=\{ \mathbb{R}^n,R_0,R_1,R_2,R_3,R_4 \}$, where $R_0 \subseteq G_1$, $R_1 \subseteq G_2$, $G_1\cup G_2 \subseteq R_2$, $R_3 \cap (G_1\cup G_2)=\emptyset$, and $R_4 \subseteq S$;
    \item $C=\{c, \hat{c}\}$ is the set of clock variables; 
    \item $\delta:S\times S\to \Theta(C) \times 2^{C}$ satisfies $\delta(s_0,s_1)=(\{c\leq a,\hat{c}\leq a\}, \{\hat{c} \})$;
        $\delta(s_1,s_2)=(\{c\leq a,\hat{c}\leq b\},\emptyset)$;
        $\delta(s_2,s_3)=(\{c\leq a,\hat{c}\leq b\},\emptyset)$;
        $\delta(s_3,s_4)=(\{\hat{c}\leq b\},\emptyset)$;
        $\delta(s_4,s_5)=(\{c\leq a,\hat{c}\leq a\},\emptyset)$;
        $\delta(s_5,s_6)=(\{c\leq a,\hat{c}\leq a\}, \{\hat{c} \})$;
         $\delta(s_6,s_7)=(\{c\leq a,\hat{c}\leq b\},\emptyset)$;
        $\delta(s_7,s_8)=(\{c\leq a,\hat{c}\leq b\},\emptyset)$;
       $\delta(s_8,s_9)=(\{\hat{c}\leq b\},\emptyset)$;
       $\delta(s_9,s_{10})=(\{c\geq a\},\emptyset)$;
    \item $M(s_0)=M(s_5)=R_3$, $M(s_1)=M(s_6)=R_2$, $M(s_2)=R_0$, $M(s_7)=R_1$, $M(s_3)=M(s_8)=M(s_{10})=\mathbb{R}^n$, $M(s_4)=M(s_9)=R_4$.
\end{itemize}
The construction is sound because the two clocks encode the two timing requirements in $\hat{\Phi}$.
The global clock $c$ is never reset: the guards $c\leq a$ force visits to $R_0\subseteq G_1$ and $R_1\subseteq G_2$ before time $a$, while the guard $c\geq a$ on the final transition keeps the implication monitored over the whole interval $[0,a]$.
Whenever the run enters a block associated with $G_1$ or $G_2$, the auxiliary clock $\hat{c}$ is reset, and the guards $\hat{c}\leq b$ force the run to reach $s_4$ or $s_9$, whose regions lie in $R_4\subseteq S$, within $b$ seconds.
In the remaining monitored states, the antecedent is false because the trajectory lies in $R_3$ disjoint from $G_1\cup G_2$, or the trajectory is already in $S$.
Hence, every accepted trajectory visits both $G_1$ and $G_2$ within $[0,a]$ and satisfies $(G_1\vee G_2)\implies \mathbf{F}_{[0,b]}S$ at all times in $[0,a]$, which implies $\xi\models\hat{\Phi}$.

Note that for STL task
\[
\hat{\Phi}_n=\mathbf{G}_{[0,a]}(\bigvee_{i=1}^nG_i \implies \mathbf{F}_{[0,b]} S) \wedge (\bigwedge_{i=1}^n \mathbf{F}_{[0,a]}G_i),
\]
the TA $A_{\hat{\Phi}_n}$ can be constructed recursively from $A_{\hat{\Phi}_{n-1}}$ by duplicating the states $s_0,s_1,s_2,s_3,s_4$ corresponding to $G_n$ and inserting the duplicated states between the penultimate and final states of $A_{\hat{\Phi}_{n-1}}$.

\end{document}